\documentclass[10pt,journal,compsoc]{IEEEtran}

\usepackage{algorithm}
\usepackage{algorithmic}

\usepackage{mathtools}
\mathtoolsset{showonlyrefs}
\usepackage[table]{xcolor}
\usepackage{multirow}
\usepackage{booktabs}
\usepackage{url}
\usepackage{enumitem}
\setlist[itemize]{leftmargin=25pt}

\ifCLASSOPTIONcompsoc
  \usepackage[nocompress]{cite}
\else
  \usepackage{cite}
\fi
\ifCLASSINFOpdf
\else
\fi
\ifCLASSOPTIONcompsoc
 \usepackage[caption=false,font=footnotesize,labelfont=sf,textfont=sf]{subfig}
\else
 \usepackage[caption=false,font=footnotesize]{subfig}
\fi

\usepackage{amsmath}
\usepackage{amsthm}
\usepackage{amsfonts}

\newtheorem{theorem}{Theorem}
\newtheorem{lemma}{Lemma}
\newcommand{\at}[2][]{#1|_{#2}}
\newtheorem{hyp}{Hypothesis}

\begin{document}

\title{ES-GNN: Generalizing Graph Neural Networks Beyond Homophily with Edge Splitting}



\author{Jingwei~Guo,
        Kaizhu~Huang*,
        Rui~Zhang,
        and~Xinping~Yi
\thanks{
The work was supported by the following: National Natural Science Foundation of China under No. 92370119, and 62376113. 
Jingwei Guo is with University of Liverpool, Liverpool, U.K. and is also with Xi’an Jiaotong-Liverpool University, Suzhou, China (E-mail:~jingwei.guo@liverpool.ac.uk); Kaizhu Huang is with Duke Kunshan University, Suzhou, China (E-mail:~kaizhu.huang@dukekunshan.edu.cn); Rui Zhang is with Xi’an Jiaotong-Liverpool University, Suzhou, China (E-mail:~rui.zhang02@xjtlu.edu.cn); Xinping Yi is with Southeast University, Nanjing, China (E-mail:~xyi@seu.edu.cn).
*Corresponding author: Kaizhu Huang.
Digital Object Identifier: 10.1109/TPAMI.2024.3459932.
}
}
\IEEEtitleabstractindextext{%

\begin{abstract}
While Graph Neural Networks (GNNs) have achieved enormous success in multiple graph analytical tasks, modern variants mostly rely on the strong inductive bias of homophily. However, real-world networks typically exhibit both homophilic and heterophilic linking patterns, wherein adjacent nodes may share dissimilar attributes and distinct labels. Therefore, GNNs smoothing node proximity holistically may aggregate both task-relevant and irrelevant (even harmful) information, limiting their ability to generalize to heterophilic graphs and potentially causing non-robustness. In this work, we propose a novel Edge Splitting GNN (ES-GNN) framework to adaptively distinguish between graph edges either relevant or irrelevant to learning tasks. This essentially transfers the original graph into two subgraphs with the same node set but complementary edge sets dynamically. Given that, information propagation separately on these subgraphs and edge splitting are alternatively conducted, thus disentangling the task-relevant and irrelevant features. Theoretically, we show that our ES-GNN can be regarded as a solution to a \textit{disentangled graph denoising problem}, which further illustrates our motivations and interprets the improved generalization beyond homophily. Extensive experiments over 11 benchmark and 1 synthetic datasets not only demonstrate the effective performance of ES-GNN but also highlight its robustness to adversarial graphs and mitigation of the over-smoothing problem.
\end{abstract}
 
}

\maketitle

\IEEEdisplaynontitleabstractindextext

%
\IEEEpeerreviewmaketitle

\IEEEraisesectionheading{\section{Introduction}\label{sec:introduction}}
\IEEEPARstart{A}{s} a ubiquitous data structure, graph can symbolize complex relationships between entities in different domains. For example, knowledge graphs describe the inter-connections between real-world events, and social networks store the online interactions between users. With the flourishing of deep learning models on graph-structured data, graph neural networks (GNNs) emerge as one of the most powerful techniques in recent years. Owing to their remarkable performance, GNNs have been widely adopted in multiple graph-based learning tasks, such as link prediction, node classification, and recommendation~\cite{ciano2021inductive,zhou2020graph,chen2020handling,zhang2020deep}. 

Modern GNNs are mainly built upon a message passing framework~\cite{gilmer2017neural}, where nodes' representations are learned by aggregating their transformed neighbors iteratively. From the graph signal denoising viewpoint, this mechanism could be seen as a low-pass filter~\cite{wu2019simplifying,balcilar2020analyzing,ma2021unified,zhu2021interpreting} that smooths the signals between adjacent nodes. Several works~\cite{zhu2020beyond,chien2021adaptive,zhu2021graph,ma2021unified,wang2021graph,suresh2021breaking,meta_yang2022graph} refer this to smoothness or homophily assumption in GNNs. Notably, they work well on homophilic (assortative) graphs, from which the proximity information of nodes can be utilized to predict their labels~\cite{nt2019revisiting}. However, real-world networks are typically abstracted from complex systems, and sometimes display heterophilic (disassortative) properties whereby the opposite objects are attracted to each other~\cite{mcpherson2001birds}. For instance, different types of amino acids are mostly interacted in many protein structures~\cite{zhu2020beyond}, and most people in heterosexual dating networks prefer to link with others of the opposite gender. Recent studies~\cite{zhu2020beyond,chien2021adaptive,zhu2021graph,fagcn2021,wang2021graph,suresh2021breaking,liu2021non,yan2021two,fang2022polarized,yang2022graph,li2022finding} have shown that the conventional neighborhood aggregation strategy may not only cause the over-smoothing problem~\cite{chen2020simple,oono2020graph} but also severely hinder the generalization performance of GNNs beyond homophily.

One reason why current GNNs perform poorly on heterophilic graphs, could be the mismatch between the labeling rules of nodes and their linking mechanism. The former is the target that GNNs are expected to learn for classification tasks, while the latter specifies how messages pass among nodes for attaining this goal. In homophilic scenarios, both of them are similar in the sense that most nodes are linked because of their commonality which therefore leads to identical labels. In heterophilic scenarios, however, the motivation underlying why two nodes get connected may be ambiguous to the classification task. Let us take the social network within a university as an example, where students from different clubs can be linked usually due to taking the same classes and/or being roommates but not sharing the same hobbies. Namely, the task-relevant and irrelevant (or even harmful) information is typically mixed into node neighborhood under heterophily. However, current methods usually fail to recognize and differentiate these two types of information within nodes' proximity, as illustrated in Fig.~\ref{fig:visual_illustration}. As a consequence, the learned representations are prone to be entangled with false information, leading to non-robustness and sub-optimal performance.

Once the issue of GNNs' learning beyond homophily is identified, a natural question arises: \textit{Can we design a new type of GNNs that is adaptive to both homophilic and heterophilic scenarios?} Well formed designs should be able to identify the node connections irrelevant to learning tasks, and substantially extract the most correlated information for prediction. However, the assortativity of real-world networks is usually agnostic.  Even worse, the features of nodes are typically full of noises, where similarity or dissimilarity between connected ones may not actually reflect their class relations. Existing techniques including~\cite{velickovic2018graph,kim2021how,fagcn2021} usually parameterize graph edges with node similarity or dissimilarity, and cannot well assess the correlation between node connections and the downstream target.

In this paper, we propose ES-GNN, an end-to-end graph learning framework that generalizes GNNs on graphs with either homophily or heterophily. Without loss of generality, we make an assumption that two nodes get connected mainly because they share some similar features, which are however unnecessarily just relevant to the learning task. 
In other words, nodes may be linked due to similar features, either relevant or irrelevant to the task. This implicitly divides the original graph edges into two complementary sets, each of which represents a latent relation between nodes. Thanks to the proximity smoothness, aggregating node features individually on each edge set should disentangle the task-relevant and irrelevant features. Meanwhile, these disentangled representations potentially reflect node similarity in two aspects (task-relevant and irrelevant). As such, they can be better utilized to split the original graph edges more precisely. Motivated by this, the proposed framework integrates GNNs with an interpretable edge splitting (ES), to jointly partition network topology and disentangle node features.

Technically, we design a residual scoring mechanism, executed within each ES-layer, to distinguish the task-relevant and irrelevant graph edges. The node features are then aggregated separately on these connections to produce disentangled representations, based on which graph edges can be classified more accurately in the next ES-layer. Finally, the task-relevant representations are granted for prediction. 
Meanwhile, an Irrelevant Consistency Regularization (ICR) is developed to regulate the task-irrelevant representations with the potential label-disagreement between adjacent nodes, for further reducing the classification-harmful information from the final predictive target.
To interpret our new algorithm theoretically, generalizing the \textit{standard smoothness assumption}~\cite{ma2021unified}, we also conduct some analysis on ES-GNN and establish its connection with a \textit{disentangled graph signal denoising problem}. 
In summary, the main contributions of this work are four-fold:
\begin{itemize}
    \item We propose a novel framework called ES-GNN for node classification tasks with one plausible hypothesis, which enables GNNs to go beyond the strong homophily assumption on graphs.
    \item We theoretically prove that our ES-GNN is equivalent to solving a graph denoising problem with a \textit{disentangled smoothness assumption}, which interprets its good performance on different types of networks.
    \item Extensive evaluations across 11 benchmark and 1 synthetic datasets illustrate ES-GNN's efficacy on graphs with varying homophily levels, achieving an average error reduction of 5.8\% over a broad spectrum of competitive methods.
    \item Importantly, ES-GNN is able to alleviate the over-smoothing problem and enjoys remarkable robustness against adversarial graphs. This shows that ES-GNN could still lead to excellent performance even if the \textit{disentangled smoothness assumption} may not hold practically.
\end{itemize}

\begin{figure}[t]
\centering
\includegraphics[clip, trim=0cm 12cm 10cm 0cm,width=0.45\textwidth]{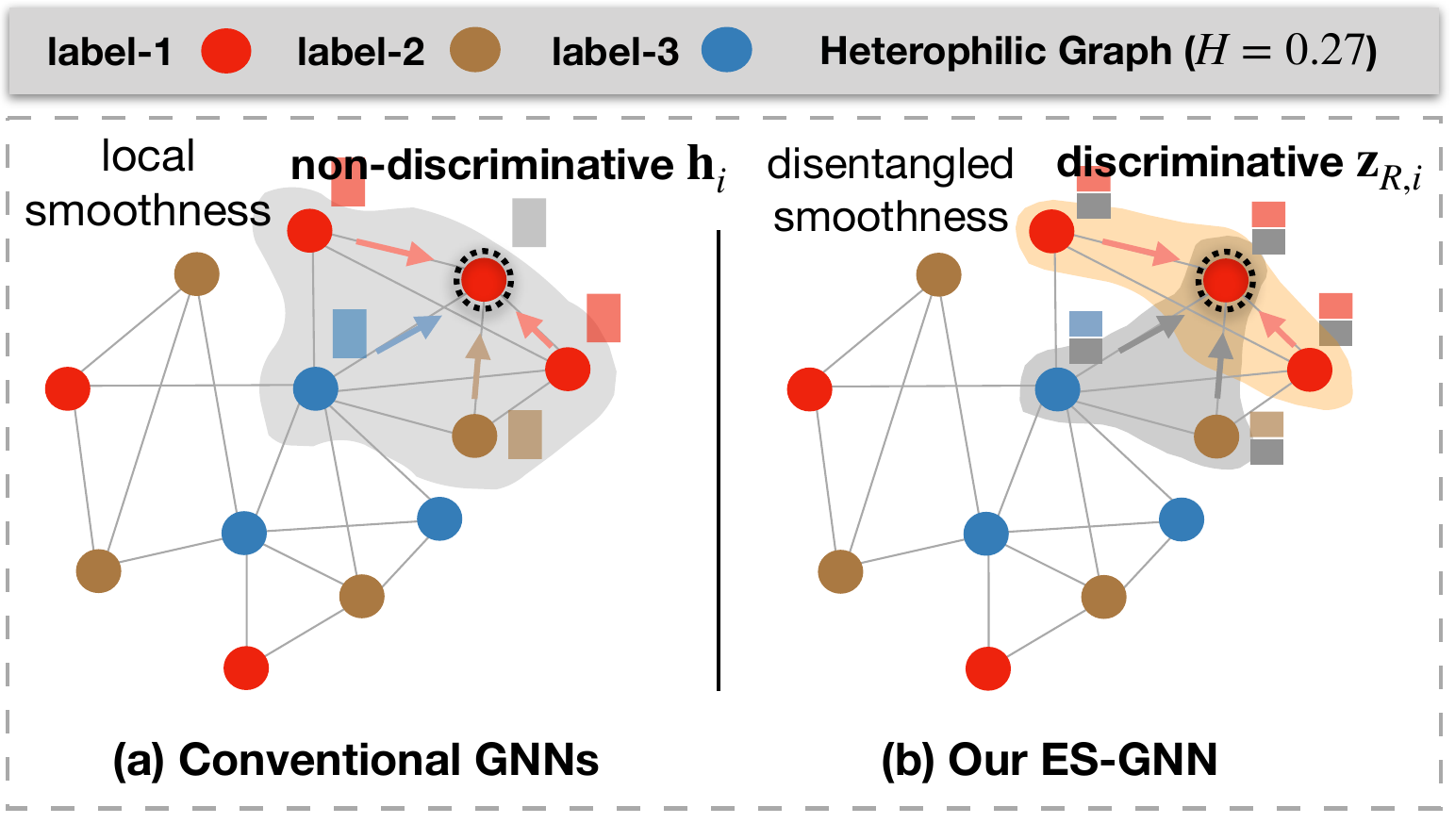}
\caption{A toy example to show differences between conventional GNNs and our ES-GNN in aggregating node features. Conventional GNNs with local smoothness tend to produce non-discriminative representations on heterophilic graphs, while our ES-GNN is able to disentangle and exclude the task-harmful features from the final predictive target.}
\label{fig:visual_illustration}
\end{figure}

\section{Preliminaries}
Let $\mathcal{G}=(\mathcal{V}, \mathcal{E})$ be an undirected graph with node set $\mathcal{V} = \{v_n\}_{n=1}^{N}$ and edge set $\mathcal{E}$, where $N = |\mathcal{V}|$ refers to node number and $(v_i, v_j) \in \mathcal{E}$ if two distinct nodes $v_i,v_j$ are connected. We use $\mathcal{N}_i$ to denote the 1-hop neighborhood of node $v_i$ and define the adjacency matrix as $\mathbf{A} \in \mathbb{R}^{N \times N}$ where $\mathbf{A}_{i,j} = 1$ if $(v_i, v_j) \in \mathcal{E}$ and 0 otherwise. 
The degree matrix $\mathbf{D}$ can be obtained by summing the row of $\mathbf{A}$ into a diagonal matrix.
As our ES-GNN disentangles the original graph into the task-relevant and irrelevant subgraphs, we will denote their adjacency matrixes respectively as $\mathbf{A}_\text{R}$ and $\mathbf{A}_\text{IR}$ in this paper. Nodes are usually associated with a feature matrix $\mathbf{X} \in \mathbb{R}^{N \times F}$ where $F$ refers to the number of raw feature and $\mathbf{X}_{[i,:]}$ is the $i$-th row of $\mathbf{X}$ pertinent to node $v_i$. For node classification tasks, each node is assigned with a label $c_i$ out of $C \leq N$ classes and have a ground truth one-hot vector $\mathbf{y}_i \in \mathbb{R}^C$.
In this context, real-world graphs can be divided into homophilic and heterophilic ones based on the extent of similarity (or dissimilarity) in class labels among connected nodes. To quantify this level of homophily, researchers have developed various metrics. Among these, edge homophily $\mathcal{H}$ is a widely used metric that calculates the proportion of edges connecting nodes with identical labels, expressed as: 
$\mathcal{H}=|\{(v_i,v_j)| \mathbf{y}_i = \mathbf{y}_j, (v_i,v_j) \in \mathcal{E}\}| / |\mathcal{E}|$.
This metric ranges from 0 (high heterophily) to 1 (high homophily). Recently, more nuanced metrics such as class homophily $\mathcal{H}_\text{class}$ and adjusted homophily $\mathcal{H}_\text{adjusted}$ have been proposed in works~\cite{lim2021large} and~\cite{platonov2024characterizing}, respectively. 
These metrics take into account potential class imbalance and the variability in class number across different datasets, offering a more accurate estimation.

\section{Background and Related Work}

In this section, we provide the necessary background and elucidate the connections between our work and previous studies in the field (see subsections~\ref{sec:connection_heterophily},~\ref{sec:connection_task} and~\ref{sec:connection_disen}).

\subsection{Graph Neural Networks}\label{sec:related_work}

The central idea of most GNNs is to utilize nodes' proximity information for building their representations for tasks, based on which great effort has been made in developing different variants~\cite{kipf2017semi,velickovic2018graph,wu2019simplifying,Klicpera2019PredictTP,xu2018representation,yang2020factorizable,policygnn_kdd20,isufi2021edgenets,bianchi2021graph,gao2022hgnn,bouritsas2022improving}, and understanding the nature of GNNs~\cite{ma2021unified,zhu2021interpreting,balcilar2021breaking,faber2021comparing,wang2022reinforced,Schnake2021HigherOrderEO}. Several works have proved that GNNs essentially behave as a low pass filter that smooths information within node surrounding~\cite{wu2019simplifying,nt2019revisiting,min2020scattering,balcilar2020analyzing}. In line with this view, \cite{zhu2021interpreting} and~\cite{ma2021unified} show that a number of GNN models, such as GCN~\cite{kipf2017semi} adopting first-order Chebyshev expansion for efficient graph convolution, SGC~\cite{wu2019simplifying}  removing non-linearity of GCN, and GAT~\cite{velickovic2018graph} parameterizing graph edges with an attention mechanism, can be seen as different optimization solvers to a graph signal denoising problem with a \textit{smoothness assumption} upon connected nodes. All these results indicate that most GNNs are designed with a strong homophily hypothesis on the observed graphs while largely overlooking the important setting of heterophily, where node features and labels vary unsmoothly on graphs. 

\subsubsection{Connection to GNNs Tailored for Heterophily}\label{sec:connection_heterophily}

This subsections briefly introduce GNNs tailored for addressing graphs under heterophilic scenarios and emphasize the differences between their approaches and ours.

To extend GNNs on heterophilic graphs, several works leverage the long-range information beyond nodes' proximity. Geom-GCN~\cite{Pei2020GeomGCNGG} extends the standard message passing with geometric aggregation in latent space. H2GCN~\cite{zhu2020beyond} directly models the higher order neighborhoods for capturing the homophily-dominant information. WRGAT~\cite{suresh2021breaking} transforms the input graph into a multi-relational graph, for modeling structural information and enhancing the assortativity level. GEN~\cite{wang2021graph} estimates a suitable graph for GNNs' learning with multi-order neighborhood information and Bayesian inference as guide. GloGNN++~\cite{li2022finding} captures the global homophily beyond immediate neighborhoods by learning a signed matrix to assess correlations among nodes. Another line of work emphasizes the proper utilization of node neighbors. The most common works employ attention mechanism~\cite{velickovic2018graph,hou2019measuring}, however, they are still imposing smoothness within nodes' neighborhood albeit on the important members only~\cite{balcilar2020analyzing,zhu2021interpreting,ma2021unified}. Compared to that, FAGCN~\cite{fagcn2021} adaptively models both similarities and dissimilarities between adjacent nodes. GPR-GNN~\cite{chien2021adaptive} introduces a universal polynomial graph filter, by associating different hop neighbors with learnable weights in both positive and negative signs, so as to extract both low- and high-frequency information. ACM-GCN~\cite{luan2022revisiting} proposes a multi-channel filtering approach that adaptively exploit both low- and high-frequency neighborhood information for each node. GOAL~\cite{zheng2023finding} enhances the modeling of intra- and inter-class node relationships through a graph complementary learning that recovers missing low- and high-frequency information in the original network topology.

However, most of them overlook the motivations why two nodes get connected, nor do they associate them with learning tasks, which is analyzed as one of the keys to generalize GNNs beyond homophily in this paper. In contrast, ES-GNN distinguishes graph edges as either relevant or irrelevant to the task. Such information acts as a guide to disentangle and exclude classification-harmful information from the final predictive target, and thus boosts GNNs' performance under heterophily.

\subsubsection{Connection to GNNs Considering Task-Relevance}\label{sec:connection_task}

Following the previous discussion, where we highlighted ES-GNN's distinctive approach of discerning task-relevant from irrelevant information amidst GNNs designed for heterophily, we now situate this idea within GNN research that focuses on task-relevance. While the notion of emphasizing task-relevant information is not new, this subsection is dedicated to clearly outlining how our work aligns with and diverges from existing methodologies in this realm.

The concept of prioritizing task-relevance in GNNs has been extensively explored across various domains, such as topological denoising~\cite{zheng2020robust,wang2020unifying,luo2021learning}, graph pooling~\cite{zhang2022dotin}, augmentations~\cite{trivedi2022augmentations,gong2023ma}, contrastive learning~\cite{xu2021infogcl}, and structure learning~\cite{sun2022graph,yang2021soft,miao2022interpretable}. Although~\cite{wang2020deep} does not explicitly define task-relevance or irrelevance, it subtly explains the reasons for node connections by employing relational topic modeling with hierarchical graph information. Given our focus on supervised classification tasks, the forthcoming discussion will be expressly centered around this area. This focus ensures a contextual analysis of ES-GNN within the established research landscape, highlighting its unique contributions to task-oriented GNN development.

One should note that there are essential distinctions between our method and the existing works of NeuralSparse~\cite{zheng2020robust} and GCN-LPA~\cite{wang2020unifying}, despite sharing the common goal of learning task-relevant edges. NeuralSparse employs a sparsification mechanism that can be trained with task loss, while GCN-LPA utilizes the outcome of label propagation as a guide to learn edge weights. Although these methods, like ours, actively select task-relevant edges to facilitate the extraction of task-relevant information, they primarily focus on this aspect, neglecting the potential benefits of modeling the opposite, task-irrelevant aspect. In contrast, our ES-GNN diverges by implementing a disentangled learning paradigm that partitions network topology and decouples node features into task-relevant and irrelevant parts. This explicit modeling of task-irrelevant information allows ES-GNN to further reduce noise and enhance the extraction of features with stronger correlations with the task. In other words, our approach not only focuses on selecting task-relevant edges but also strategically minimizes the impact of irrelevant information, therefore excelling in complex settings like heterophilic graphs (see Table~\ref{tab:real_nc}).

Additionally, while both of our work and DOTIN~\cite{zhang2022dotin} explore the enhancement of GNNs by identifying task-irrelevant graph information, the approaches we take diverge in their application and focus. While DOTIN focuses on streamlining graph classification by dropping task-irrelevant nodes to boost efficiency and scalability, our ES-GNN targets node-level tasks, emphasizing the discernment of task-relevance in graph edges to generalize GNNs beyond homophily. Extending the core principles of ES-GNN to graph-level tasks presents a promising future direction.

\subsection{Disentangled Representation Learning}

Disentangled representation learning, aimed at disentangling the explanatory latent variables within observed data into distinct dimensions~\cite{bengio2013representation,higgins2018towards}, has garnered considerable attention, particularly in the field of computer vision~\cite{tang2023context,dalva2023image,chu2021learning,chen2023sketchtrans,liu2022spoof}. In recent years, there has been a progressive expansion of disentangled representation learning into the graph domain, addressing a wide spectrum of tasks. These range from foundational classification~\cite{ma2019disentangled,liu2020independence,yang2020factorizable,li2021disentangled,xiao2022decoupled,li2022disentangled,he2022variational,zhao2022exploring,guo2022learning,fan2022debiasing,zheng2023adversarial,wu2022multi} and generation challenges~\cite{guo2020interpretable,guo2021deep,wang2022deep,mercatali2022symmetry,du2022disentangled}, in both supervised and unsupervised settings, to downstream applications like trajectory prediction~\cite{bae2021disentangled}, overlapping community detection~\cite{zhou2015infinite}, recommendation systems~\cite{xia2023disentangled,ren2023disentangled,li2023edge,li2022disenRecommendation,zhao2022multi}, and graph neural architecture search~\cite{qin2022graph}. Given our focus on foundational Graph Neural Network (GNN) models, the following discussion will concentrate on the subset of research that employs disentangled representation learning to enhance the capabilities of GNNs within the realm of fundamental tasks.

For instance, DisenGCN~\cite{ma2019disentangled} introduces a neighborhood routing mechanism to iteratively partition node neighborhoods into distinct segments, paving the way for disentangled node-level information learning. Following this, IPGDN~\cite{liu2020independence} and LGD-GCN~\cite{guo2022learning} further enhance the model by promoting independence among disentangled factors and integrating global graph information, respectively. At the graph-level, FactorGCN~\cite{yang2020factorizable} takes a novel approach by factorizing the original graph into multiple subgraphs, aiming to highlight various graph aspects. VEPM~\cite{he2022variational} subsequently extends this learning paradigm by developing an edge generative model that incorporates community information to partition edges. Distinct from the aforementioned works, DisGNN~\cite{zhao2022exploring} focuses explicitly on disentangling graph edges. It employs three pretext tasks to guide the learning process, aiming to enhance GNN performance under heterophily settings -- a goal that aligns closely with our work. 

Shifting the focus to unsupervised learning approaches, DGCL~\cite{li2021disentangled} introduces a factor-wise discrimination objective in a contrastive learning manner to disentangle graph-level representations. Building upon this foundation, IDGCL~\cite{li2022disentangled} further enhances this approach by promoting the independence among the disentangled latent representations. Complementing these at the node level, DSSL~\cite{xiao2022decoupled} advances graph self-supervised learning by simulating a graph generative process through latent variable modeling of semantic structures. This process effectively decouples diverse neighborhood contexts, particularly benefiting the analysis of non-homophilous graphs.
In the realm of graph generation, NED-VAE~\cite{guo2020interpretable} stands out as a unsupervised approaches by automatically disentangling latent factors in both nodes and edges. SND-VAE~\cite{guo2021deep} further advances this field as the first disentangled generative model tailored for spatial networks. It adeptly uncovers both independent and dependent latent factors of spatial and network domains.

\subsubsection{Connections to Disentangled GNNs}\label{sec:connection_disen}

This subsection explores the relationships between our ES-GNN model and established disentangled GNNs, such as FactorGCN~\cite{yang2020factorizable}, VEPM~\cite{he2022variational} and DisGNN~\cite{zhao2022exploring}, specifically within the context of supervised settings. Our emphasis on supervised node classification tasks guides the selection of these comparative models to highlight the unique contributions and distinctions of our approach.

First, we acknowledge that our work shares a foundational similarity with both FactorGCN and VEPM: the aim to decompose the original network topology into multiple subgraphs for disentangling node features. However, there are three main differences: \textbf{1)} unlike FactorGCN, which allows an edge to belong to multiple subgraphs, resulting in potential overlap, our ES-GNN adopts an edge-splitting strategy that adaptively divides the original network topology into two mutually complementary subgraphs, ensuring $\mathbf{A}_\text{R}+\mathbf{A}_\text{IR}=\mathbf{A}$. In this aspect, VEPM is somewhat similar to ours, also producing complementary subgraphs by normalizing edge weights with a softmax layer. \textbf{2)} FactorGCN merely interprets the decomposed subgraphs as different graph aspects without providing any concrete meanings, and the number of latent factors requires manual selection across different graphs. While VEPM attributes community characteristics to these subgraphs, it falls short in linking these characteristics directly to the task at hand, nor does this approach address the variability in the community number needed across graphs. In contrast, our model uniquely generates two interpretable, task-relevant and irrelevant topologies adaptable to any graph, offering more meaningful and application-specific insights. \textbf{3)} FactorGCN and VEPM integrate all disentangled components towards the final prediction, with VEPM even remixing the disentangled feature representations for prediction using a ``representation composer’’. Diverging from them, our ES-GNN focuses on segregating task-relevant from task-irrelevant features, allowing for the exclusion of classification-harmful information in the predictive process. This distinction is particularly beneficial in heterophilic contexts, where task-irrelevant information could easily obscure the target prediction, as empirically validated in our experiments (see Table~\ref{tab:real_nc}).

Second, as previously mentioned when introducing DisGNN, both our method and DisGNN aim to enhance GNN performance on heterophilic graphs through explicit edge disentanglement. However, unlike DisGNN, which relies on multiple heuristic-based pretext tasks to supervise the edge disentanglement process, our approach requires only the addition of an Irrelevance Consistency Regularization (ICR) loss alongside the main task loss. This ICR loss, systematic in nature, adheres strictly to our core model principle as outlined in Hypothesis~\ref{hyp:our}. Moreover, similar to FactorGCN and VEPM, DisGNN does not prioritize task relevance when utilizing disentangled components for prediction. This approach risks retaining misleading information in heterophilic scenarios and potentially compromises model performance.

\section{Framework: ES-GNN}

In this section, we propose an end-to-end graph learning framework, ES-GNN, generalizing Graph Neural Networks (GNNs) to arbitrary graph-structured data with either homophilic or heterophilic properties. An overview of ES-GNN is given in Fig.~\ref{fig:gf_pipline}. The central idea is to integrate GNNs with an interpretable edge splitting (ES) layer that adaptively partitions the network topology as guide to disentangle the task-relevant and irrelevant node features.

\subsection{Edge Splitting Layer}\label{sec:es_layer}
The goal of this layer is to infer the latent relations underlying adjacent nodes on the observed graph, and distinguish between graph edges which could be relevant or irrelevant to learning tasks. 
Given a simple graph with an adjacency matrix $\mathbf{A}$ and node feature matrix $\mathbf{X}$, an ES-layer splits the original graph edges into two complementary sets, and thereby produces two partial network topologies with adjacency matrices $\mathbf{A}_{\text{\text{R}}},\mathbf{A}_{\text{\text{IR}}} \in \mathbb{R}^{N \times N}$ satisfying $\mathbf{A}_{\text{\text{R}}} + \mathbf{A}_{\text{\text{IR}}} = \mathbf{A}$. We would expect $\mathbf{A}_{\text{R}}$ storing the most correlated graph edges to the classification task, of which the rest is excluded and disentangled in $\mathbf{A}_{\text{IR}}$. Therefore, analyzing the correlation between node connections and learning tasks comes into the first step.

However, existing techniques~\cite{velickovic2018graph,kim2021how,fagcn2021} mainly parameterize graph edges with node similarity or dissimilarity, while failing to explicitly correlate them with the prediction target. Even worse, as the assortativity of real-world networks is usually agnostic and node features are typically full of noises, the captured similarity/dissimilarity may not truly reflect the label-agreement/disagreement between nearby nodes. Consequently, the harmful-similarity between pairwise nodes from different classes could be mistakenly preserved for prediction. 
To this end, we present one plausible hypothesis below, whereby the explicit correlation between node connections and learning tasks is established automatically.

\begin{hyp}\label{hyp:our}
Two nodes get connected in a graph mainly due to their similarity in some features, which could be either relevant or irrelevant (even harmful) to the learning task. 
\end{hyp}

This hypothesis is assumed without losing generality to both homophilic and heterophilic graphs. For a homophilic scenario, e.g., in citation networks, scientific papers tend to cite or be cited by others from the same area, and both of them usually possess the common keywords uniquely appearing in their topics. For a heterophilic scenario, students having different interests are likely be connected because of the same classes and/or dormitory they take and/or live in, but neither has direct relation to the clubs they have joined. This inspires us to classify graph edges by measuring the similarity between adjacent nodes in two different aspects, i.e., \textit{a graph edge is more relevant to a classification task if connected nodes are more similar in their task-relevant features, or otherwise.} Our experimental analysis in Section~\ref{sec:robust} further provides evidences that even when our Hypothesis~\ref{hyp:our} may not hold, most adversarial edges (considered as the task-irrelevant ones)  can still be recognized though neither types of node similarity exists.

\begin{figure}[!t]
\centering
\includegraphics[clip, trim=0cm 12.5cm 10cm 0cm,width=0.49\textwidth]{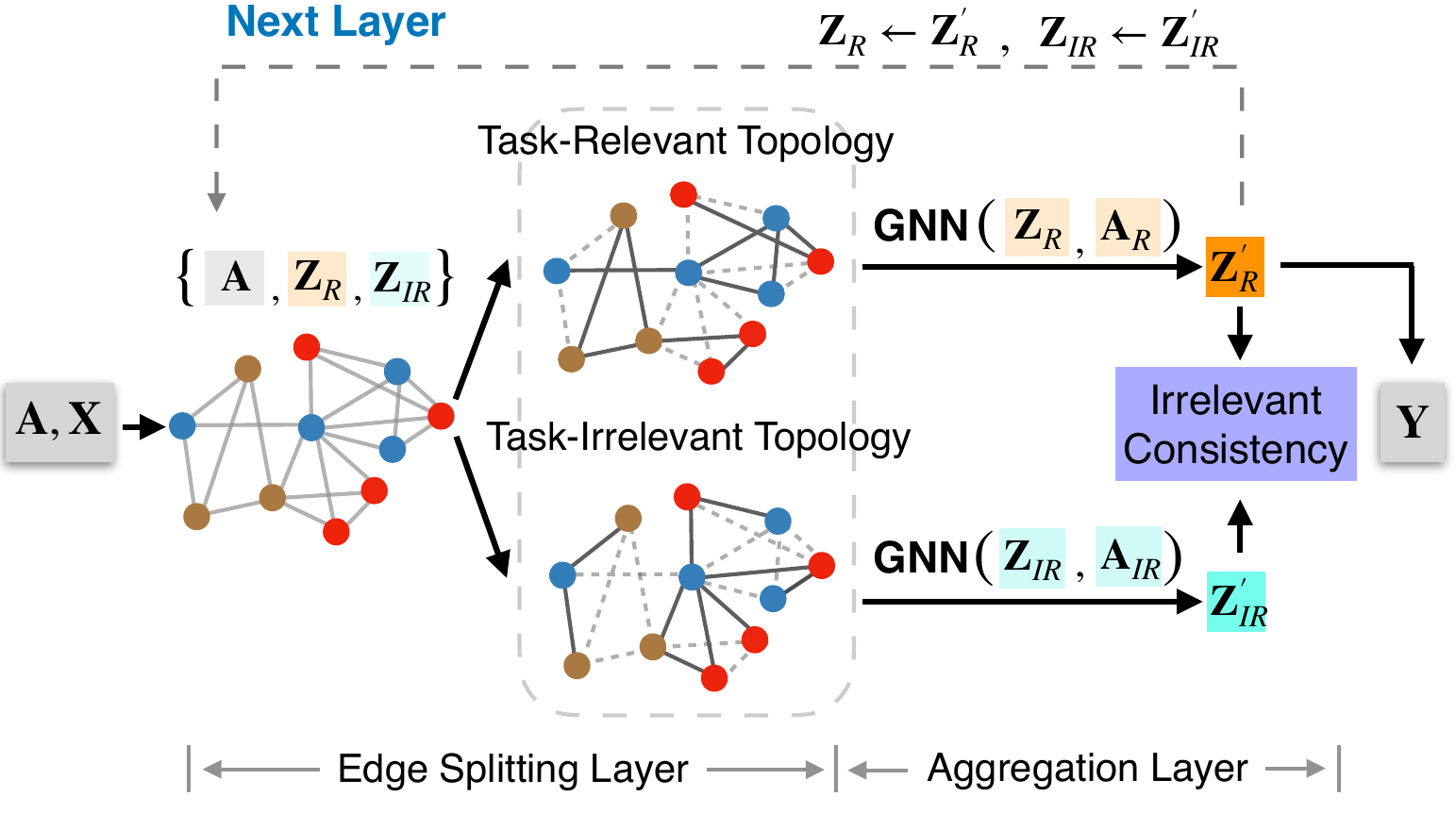}
\caption{Illustration of ES-GNN framework where $\mathbf{A}$ and $\mathbf{X}$ denote the adjacency matrix and feature matrix of nodes, respectively. First, $\mathbf{X}$ is projected onto different latent subspaces via different channels $\text{R}$ and $\text{IR}$. An edge splitting is then performed to divide the original graph edges into two complementary sets. After that, the node information can be aggregated individually and separately on different edge sets to produce disentangled representations, which are further utilized to make an more accurate edge splitting in the next layer. The task-relevant representation $\mathbf{Z}^{'}_{\text{R}}$ is reasonably granted for prediction, and an Irrelevant Consistency Regularization (ICR) term is developed to further reduce the potential task-harmful information from the final predictive target.}
\label{fig:gf_pipline}
\end{figure}

It is worthy mentioning that our hypothesis is not in contradiction to the ``opposites attract'', which could be intuitively explained by linking due to different but matching attributes.
We believe the inherent cause to connection even in ``opposites attract'' may still be certain commonalities. For example, in heterosexual dating networks, people of the opposite sex are most likely connected because of their similar life values. Although these similarities may be inappropriate (or even harmful) in distinguishing genders, modeling and disentangling them from the final predictive target might be still of great importance.

An ES-layer consists of two channels to respectively extract the task-relevant and irrelevant information from nodes. As only the raw feature matrix $\mathbf{X}$ is provided in the beginning, we will project them into two different subspaces before the first ES-layer:
\begin{equation}
\mathbf{Z}_s^{(0)} = \sigma(\mathbf{W}_s^T\mathbf{X} + \mathbf{b}_s),
\label{eq:channels_mp}
\end{equation}
where $\mathbf{W}_s \in \mathbb{R}^{f \times \frac{d}{2}}$ and $\mathbf{b}_s \in \mathbb{R}^{\frac{d}{2}}$ are the learnable parameters in channel $s \in \{\text{R}, \text{IR}\}$, $d$ is the number of node hidden states, and $\sigma$ is a nonlinear activation function. 

Given Hypothesis~\ref{hyp:our}, we adopt a flexible approach for classifying node connections by using continuous edge weights from 0 to 1, reflecting the varying degrees to which edges are task-relevant or irrelevant. Nevertheless, applying metrics to independently determine $\mathbf{A}_{\text{R}}$ and $\mathbf{A}_{\text{IR}}$ based on node similarity may not fully capture the complex interplay between different channels and could diminish the focus on topological distinctions. To address this, for edges where $\mathbf{A}_{(i,j)}=1$, we parameterize the difference between $\mathbf{A}_{\text{R}(i,j)}$ and $\mathbf{A}_{\text{IR}(i,j)}$, by solving the linear equation:
\begin{equation} \label{eq:linear_eq}
\begin{cases}
\mathbf{A}_{\text{R}(i,j)} - \mathbf{A}_{\text{IR}(i,j)} = \alpha_{i,j}\\ 
\mathbf{A}_{\text{R}(i,j)} + \mathbf{A}_{\text{IR}(i,j)} = 1
\end{cases}.
\end{equation}
This gives us $\mathbf{A}_{\text{R}(i,j)} = \frac{1 + \alpha_{i,j}}{2}$ and $\mathbf{A}_{\text{IR}(i,j)} = \frac{1 - \alpha_{i,j}}{2}$ with $-1 \leq \alpha_{i,j} \leq 1$. 
To effectively quantify the interaction (or relative importance) between the task-relevant and irrelevant aspects of each edge, we propose a residual scoring mechanism:
\begin{equation} \label{eq:rs_eq}
\alpha_{i,j} = \tanh(\mathbf{g}\left[\mathbf{Z}_{\text{R}[i,:]} \oplus \mathbf{Z}_{\text{IR}[i,:]} \oplus \mathbf{Z}_{\text{R}[j,:]} \oplus \mathbf{Z}_{\text{IR}[j,:]} \right]^T).
\end{equation}
Here, both of the task-relevant and irrelevant node features are first concatenated and convoluted by learnable $\mathbf{g} \in \mathbb{R}^{1 \times 2d}$, and then passed to the tangent activation function to produce a floating value between -1 and 1. Similar learning scheme can be found in works~\cite{velickovic2018graph,fagcn2021,kim2021how}. To further enhance the distinction between $\mathbf{A}_\text{R}$ and $\mathbf{A}_\text{IR}$, while acknowledging their inherent continuous nature, one can apply techniques, such as softmax with temperature in Eq.~\eqref{eq:sft_t}, Gumbel-Softmax~\cite{jang2016categorical,maddison2016concrete} in Eq.~\eqref{eq:gumbel}, or thresholding in Eq.~\eqref{eq:thresh}. These methods aim to bring their values closer to 0 or 1, thereby strengthening the clarity of task relevance and promoting graph disentanglement.

\begin{align}
\mathbf{A}_{\textit{s}(i,j)}^{'} &= \frac{\exp(\mathbf{A}_{\textit{s}(i,j)} / \tau)}{\sum_{\kappa \in \{\text{R},\text{IR}\}}\exp(\mathbf{A}_{\kappa(i,j)} / \tau)}\label{eq:sft_t}\\
\mathbf{A}_{\textit{s}(i,j)}^{'} &= \frac{\exp((\log(\mathbf{A}_{\textit{s}(i,j)}) + \gamma)/ \tau)}{\sum_{\kappa \in \{\text{R},\text{IR}\}}\exp((\log(\mathbf{A}_{\kappa(i,j)}) + \gamma)/ \tau)}\label{eq:gumbel}\\
\mathbf{A}_{\textit{s}(i,j)}^{'} &= \begin{cases}
1 & \mathbf{A}_{\textit{s}(i,j)} > 0.5 \\
0 & \text{otherwise}
\end{cases}
\label{eq:thresh}
\end{align}
where $s \in \{\text{R}, \text{IR}\}$, $\tau$ is a hyper-parameter mediating discreteness degree, and $\gamma \sim \textit{Gumbel}(0, 1)$ is a Gumbel random variable. 
However, in this work, we find good results without adding any additional discretization techniques, and will leave this investigation to the future work.

\subsection{Aggregation Layer}\label{sec:agg_ly}
As the split network topologies disclose the partial relations among nodes in different latent spaces, they can be utilized  to aggregate information for learning different node aspects. Specifically, we leverage a simple low-pass filter with scaling parameters $\{\epsilon_\text{R}, \epsilon_\text{IR}\}$ for both task-relevant and irrelevant channels, from the $k$-th to $k+1$-th layer:
\begin{equation} \label{eq:agg_eq}
\mathbf{Z}_s^{(k+1)} = \epsilon_s\mathbf{Z}_s^{(0)} + (1-\epsilon_s) \mathbf{D}^{-\frac{1}{2}}_s \mathbf{A}_s \mathbf{D}^{-\frac{1}{2}}_s \mathbf{Z}_s^{(k)}.
\end{equation}
$s \in \{\text{R}, \text{IR}\}$ denotes the task-relevant or irrelevant channel, and $\mathbf{D}_s$ is the degree matrix associated with the adjacency matrix $\mathbf{A}_s$. Derivation of Eq.~\eqref{eq:agg_eq} is detailed in our theoretical analysis. Importantly, by incorporating proximity information in different structural spaces, the task-relevant and irrelevant information can be better disentangled in $\mathbf{Z}_{\text{\text{R}}}^{(k+1)}$ and $\mathbf{Z}_{\text{IR}}^{(k+1)}$, based on which the next ES-layer
can make a more precise partition on the raw topology. 

\begin{algorithm}[t]
\caption{Framework of ES-GNN}
\label{alg:es_gnn}
\begin{algorithmic}[1]
\REQUIRE nodes set: $\mathcal{V}$, edge set: $\mathcal{E}$, adjacency matrix: $\mathbf{A} \in \mathbb{R}^{N \times N}$, node feature matrix: $\mathbf{X} \in \mathbb{R}^{|V| \times F}$, the number of layers: $K$, scaling parameters: $\{\epsilon_\text{R},\epsilon_\text{IR}\}$, irrelevant consistency coefficient: $\lambda_\text{ICR}$, and ground truth labels on the training set: $\{\mathbf{y}_i \in \mathbb{R}^C| \forall v_i \in \mathcal{V}_\text{trn}\}$.
\ENSURE $\mathbf{W}_\text{R}, \mathbf{W}_\text{IR} \in \mathbb{R}^{f \times d}$, $\mathbf{W}_{F} \in \mathbb{R}^{d \times C}, \mathbf{b}_{F} \in \mathbb{R}^C$, $\{\mathbf{g}^{(k)} \in \mathbb{R}^{1 \times 2d}|k=0,1,...,K-1\}$
\STATE // \textit{Project node features into two subspaces}.
\FOR{$s \in \{\text{\text{R}}, \text{\text{IR}}\}$}
    \STATE $\mathbf{Z}_{s}^{(0)} \gets \sigma(\mathbf{W}_s^T\:\mathbf{X} + \mathbf{b}_s)$.
    \STATE $\mathbf{Z}_{s}^{(0)} \gets \text{Dropout}(\mathbf{Z}_{s}^{(0)})$ // \textit{Enabled only for training}.
\ENDFOR
\STATE // \textit{Stack Edge Splitting and Aggregation Layers}.
\FOR{layer number $k=0,1,...,K-1$}
    \STATE // \textit{Edge Splitting Layer}.
    \STATE Initialize $\mathbf{A}_\text{R}, \mathbf{A}_\text{IR} \in \mathbb{R}^{N \times N}$ with zeros.
    \FOR{$\left(v_i,v_j\right) \in \mathcal{E}$}
        \STATE \small $\alpha_{i,j} \gets \tanh({\mathbf{g}^{(k)}}\left[\mathbf{Z}_{{\text{\text{R}}}[i,:]}^{(k)} \oplus \mathbf{Z}_{{\text{\text{IR}}}[i,:]}^{(k)} \oplus \mathbf{Z}_{{\text{\text{R}}}[j,:]}^{(k)} \oplus \mathbf{Z}_{\text{IR}, [j,:]}^{(k)} \right]^T)$.
        \STATE $\alpha_{i,j} \gets \text{Dropout}(\alpha_{i,j})$ // \textit{Enabled only for training}.
        \STATE $\mathbf{A}_{\text{R} (i,j)} \gets \frac{1 + \alpha_{i,j}}{2}$, $\mathbf{A}_{\text{IR} (i,j)} \gets \frac{1 - \alpha_{i,j}}{2}$.
    \ENDFOR
    \STATE // \textit{Aggregation Layer}.
    \FOR{$s \in \{{\text{\text{R}}}, \text{\text{IR}}\}$}
        \STATE $\mathbf{Z}_s^{(k+1)} \gets \epsilon_s\mathbf{Z}_s^{(0)} + (1 - \epsilon_s) \mathbf{D}^{-\frac{1}{2}}_s \mathbf{A}_s \mathbf{D}^{-\frac{1}{2}}_s \mathbf{Z}_s^{(k)}$.
    \ENDFOR
\ENDFOR
\STATE // \textit{Prediction}.
\STATE $\hat{\mathbf{y}}_i = \text{softmax}(\mathbf{W}_F^T\mathbf{Z}_{{\text{\text{R}}}[i,:]}^{(K)} + \mathbf{b}_F), \forall v_i \in \mathcal{V}$.
\STATE // \textit{Optimization with Irrelevant Consistency Regularization}.
\STATE $\mathcal{L}_{\text{ICR}} = \sum_{(v_i, v_j) \in \mathcal{E}} (1 - \delta(\hat{\mathbf{y}}_i, \hat{\mathbf{y}}_j)) \|\mathbf{Z}_{\text{\text{IR}}[i,:]} - \mathbf{Z}_{\text{\text{IR}}[j,:]}\|_2^2$.
\STATE $\mathcal{L}_{\text{pred}} = -\frac{1}{|\mathcal{V}_{\text{trn}}|}\sum_{i \in \mathcal{V}_{\text{trn}}}\mathbf{y}_i^T\log(\hat{\mathbf{y}}_i)$.
\STATE Minimize $\mathcal{L}_{\text{pred}} + \lambda_{\text{ICR}}\mathcal{L}_{\text{ICR}}$.
\end{algorithmic}
\end{algorithm}

\subsection{Irrelevant Consistency Regularization}

Stacking ES-layer and aggregation layer iteratively lends itself to disentangling different features of nodes into two distinct representations, denoted by $\mathbf{Z}_{\text{\text{R}}}$ and $\mathbf{Z}_{\text{\text{IR}}}$. First, $\mathbf{Z}_{\text{R}}$, informed and shaped by $\mathbf{A}_{\text{R}}$, is tuned for prediction, with its development guided by the minimization of the classification loss $\mathcal{L}_{\text{pred}}$. This process not only makes $\mathbf{Z}_{\text{R}}$ predictive of node labels but also implicitly reinforces the task-relevant nature of $\mathbf{A}_{\text{R}}$ via message passing. However, only supervising one channel ($\text{R}$) risks neglecting the meaningfulness of the other ($\text{IR}$), potentially leading to the preservation of erroneous information in predictions. To this end, we introduce a Irrelevant Consistency Regularization (ICR) loss $\mathcal{L}_{\text{ICR}}$, designed to regulate $\mathbf{Z}_{\text{IR}}$ as the opposite of $\mathbf{Z}_{\text{R}}$, i.e., identifying the classification-harmful information within the observed graph. The key rationale is to explore the similarities among nodes that are detrimental to classification tasks within $\mathbf{Z}_{\text{\text{IR}}}$. Given any node pairs $(v_i,v_j) \in \mathcal{E}$, we would expect $\mathbf{Z}_{{\text{\text{IR}}}[i,:]}$ and $\mathbf{Z}_{{\text{\text{IR}}}[j,:]}$ to be close in the latent space if they possess different labels, which can be formulated as:
\begin{equation} \label{eq:icr}
\mathcal{L}_{\text{ICR}} = \sum_{(v_i, v_j) \in \mathcal{E}} (1 - \delta(\mathbf{y}_i,\mathbf{y}_j)) \|\mathbf{Z}_{\text{\text{IR}}[i,:]} - \mathbf{Z}_{\text{\text{IR}}[j,:]}\|_2^2,
\end{equation}
where $\delta$ is a Kronecker function (1 if $\mathbf{y}_i=\mathbf{y}_j$, 0 otherwise) and $\|\cdot\|_2$ denotes $L_2$ norm. 
As such, $\mathbf{Z}_{{\text{\text{IR}}}}$ is constrained with a local consistency between adjacent nodes from different classes,
aiding in the exclusion and disentanglement of classification-harmful information from $\mathbf{Z}_{\text{\text{R}}}$ and to $\mathbf{Z}_{\text{\text{IR}}}$.

Several powerful techniques have been developed to assess label agreement between nodes~\cite{Stretcu2019GraphAM,kim2021how}. In this work, however, we find that using the joint probability from model predictions is effective and eliminates the need for additional trainable parameters. Besides, while the idea in ICR could be adapted to task-relevant representations by making closer nodes with the same predicted label, we avoid this due to the risk of inaccurate prediction during training. Such inaccuracy could irreversibly distort the classification metric space. Instead, we supervise task-irrelevant representations -- unused for direct prediction -- with noisy labels, offering a margin of error that our proposed layers can correct before reaching the final prediction stage.

\subsection{Overall Algorithm}
The overall pipeline of ES-GNN is detailed in Algorithm~\ref{alg:es_gnn}. Specifically, we adopt ReLU activation function in Eq.~\eqref{eq:channels_mp} to first map node features into two different channels, and then pass them with the adjacency matrix to an ES-layer for splitting the raw network topology into two complementary parts. After that, these two partial network topologies are utilized to aggregate information in different structural spaces. 
Alternatively stacking ES-layer and aggregation layer not only enables more accurate disentanglement but also explores the graph information beyond local neighborhood. 
Finally, a fully connected layer is appended to project the learned representations into class space $\mathbb{R}^C$. We integrate $\mathcal{L}_{\text{ICR}}$ into the optimization process with a irrelevant consistency coefficient $\lambda_{\text{ICR}}$ to have final objective function below, where $\mathcal{L}_{\text{pred}} = -\frac{1}{|\mathcal{V}_{\text{trn}}|}\sum_{v_i \in \mathcal{V}_\text{trn}}\mathbf{y}_i^T\log(\hat{\mathbf{y}}_i)$.
\begin{equation}\label{eq:final_op}
\mathcal{L} = \mathcal{L}_{\text{pred}} + \lambda_{\text{ICR}} \mathcal{L}_{\text{ICR}}.
\end{equation}

It is noted that the method ES-GNN employs in Eq.~(\ref{eq:rs_eq}) for learning edge weights diverges from the $L_2$ space metrics used in our ICR loss. While parameterizing edges with node similarity in $L_2$ space seems straightforward, this method is only feasible for modeling task-relevant and irrelevant channels independently. Such an approach may not fully capture the intricate interactions across different channels, possibly reducing the focus on topological distinctions. Alternatively, our strategy employs a flexible attention mechanism, unlike direct metric computation, allowing for the nuanced learning of weight residuals between different channels. Importantly, the use of learnable $\mathbf{g}$ in Eq.~(\ref{eq:rs_eq}) lends our method a universal fitting capability, enabling it to adapt and bridge potential inconsistencies between different framework components. This flexibility ensures that our model effectively integrates and responds to the diverse dynamics within the graph structure, maintaining our focus on graph disentanglement.

Finally, we also report in Table~\ref{tab:time_comp} the complexity of our ES-GNN in comparison with the baseline models evaluated in the experimental section. Clearly, our model displays the same complexity to FAGCN~\cite{fagcn2021} while being slightly overhead compared to GPR-GNN~\cite{chien2021adaptive}. Here, we omit the related works, such as GEN~\cite{wang2021graph}, WRGAT~\cite{suresh2021breaking}, GloGNN++~\cite{li2022finding}, ACM-GCN~\cite{luan2022revisiting}, and GOAL~\cite{zheng2023finding} as their complexity is obviously higher than others by involving graph reconstruction or node-wise operations.

\section{Theoretical Analysis}
In this section, we investigate two important problems: (1) what limits the generalization power of the conventional GNNs on graphs beyond homophily, and (2) how the proposed ES-GNN breaks this limit and performs well on different types of networks. We will answer these questions by first analyzing the typical GNNs as graph signal denoising from a more generalized viewpoint, and then impose our Hypothesis~\ref{hyp:our} to derive ES-GNN.

\subsection{Limited Generalization  of Conventional GNNs}
Recent studies~\cite{zhu2021interpreting,ma2021unified} have proved that most GNNs can be regarded as solving a graph signal denoising problem:
\begin{equation}
    \mathop{\arg \min }_{\mathbf{Z}}\ \ \|\mathbf{Z} - \mathbf{X}\|^2_2 + \xi \cdot tr(\mathbf{Z}^T\mathbf{L}\mathbf{Z}), \label{eq:1st_gsd}
\end{equation}
where $\mathbf{X} \in \mathbb{R}^{N \times F}$ is the input signal, $\mathbf{L} = \mathbf{D} - \mathbf{A} \in \mathbb{R}^{N\times N}$ is the graph Laplacian matrix, and $\xi$ is a constant coefficient. The first term guides $\mathbf{Z}$ to be close to $\mathbf{X}$, while the second term $tr(\mathbf{Z}^T\mathbf{L}\mathbf{Z})$ is the Laplacian regularization, enforcing smoothness between connected nodes. One fundamental assumption made here is that similar nodes should have a higher tendency to connect each other, and we refer it as \textit{standard smoothness assumption} on graphs. However, real-world networks typically exhibit diverse linking patterns of both assortativity and disassortativity. Constraining smoothness on each node pair is prone to mistakenly preserve both of the task-relevant and irrelevant (or even harmful) information for prediction. Given that, we divide the original graph into two subgraphs with the same nodes sets but complementary edge sets, and reformulate Eq.~\eqref{eq:1st_gsd} as:
\begin{equation}\label{eq:2nd_gsd}
\centering
\begin{split}
\mathop{\arg \min }_{\mathbf{Z}}\ \  \|\mathbf{Z} - \mathbf{X}\|^2_2 + \xi \cdot tr(\mathbf{Z}^T\mathbf{L}_{\text{\text{R}}}\mathbf{Z}) + \xi \cdot tr(\mathbf{Z}^T\mathbf{L}_{\text{\text{IR}}}\mathbf{Z}).
\end{split}
\end{equation}
Here, $\mathbf{L}_{\text{\text{R}}} = \mathbf{D}_{\text{\text{R}}} - \mathbf{A}_{\text{\text{R}}}$, and $\mathbf{L}_{\text{\text{IR}}} = \mathbf{D}_{\text{\text{IR}}} - \mathbf{A}_{\text{\text{IR}}}$, where the task-relevant and irrelevant node relations are separately captured in $\mathbf{A}_{\text{\text{R}}}$ and $\mathbf{A}_{\text{\text{IR}}}$. Clearly, emphasizing the commonality between adjacent nodes in $\mathbf{A}_{\text{\text{R}}}$ is beneficial for keeping task-correlated information only. However, smoothing node pairs in $\mathbf{A}_{\text{\text{IR}}}$ simultaneously may preserve classification-harmful similarity between nodes, thus limiting the prediction performance of GNNs.

\begin{table}[t]\caption{Time complexity of the comparison models with one hidden layer as an example. $N_e$ denotes the number of graph aspects assumed in FactorGCN~\cite{yang2020factorizable}, $D_\text{max}$ represents the maximum node degree, and $|\mathcal{E}_2|$ is the total number of neighbors in the second hop of nodes. Other symbols are earlier defined in the texts.}\label{tab:time_comp}
\centering
\resizebox{0.4\textwidth}{!}{
\begin{tabular}{l|l}
\toprule
\textbf{Models}   &\textbf{Complexity} \\ \midrule
GCN~\cite{kipf2017semi}  &$\mathcal{O}((f+C)|\mathcal{E}|d)$ \\
GAT~\cite{velickovic2018graph}  &$\mathcal{O}(((2+f)N + (4+C)|\mathcal{E}|)d)$ \\
FactorGCN~\cite{yang2020factorizable}  &$\mathcal{O}(N_e N+(Nf+(3+C)|\mathcal{E}|)d)$ \\
H2GCN~\cite{zhu2020beyond}  &$\mathcal{O}(fd + |\mathcal{E}|D_\text{max} + (|\mathcal{E}| + |\mathcal{E}_2|)d)$ \\
FAGCN~\cite{fagcn2021} &$\mathcal{O}(((1+C+f)N + |\mathcal{E}|)d)$ \\
GPR-GNN~\cite{chien2021adaptive} &$\mathcal{O}((fN + |\mathcal{E}|C)d)$ \\
\textbf{ES-GNN (Ours)}  &$\mathcal{O}(((1+C+f)N + |\mathcal{E}|)d)$ \\
\bottomrule
\end{tabular}
}
\end{table}

\subsection{Disentangled Smoothness Assumption in ES-GNN}
Our Hypothesis~\ref{hyp:our} suggests that the original graph topology can be partitioned into two complementary ones, wherein connected nodes displays high similarity with either task-relevant or irrelevant features only. We further interpret this result as \textit{disentangled smoothness assumption}, based on which the conventional graph signal denoising problem in Eq.~\eqref{eq:1st_gsd} can be generalized as:
\begin{equation}\label{eq:disen_gsd}
\begin{aligned}
\mathop{\arg \min}_{\mathbf{Z}_{\text{R}},\mathbf{Z}_{\text{IR}}} \quad 
                    & \|\mathbf{Z}_{\text{R}} - \mathbf{X}_{\text{R}}\|^2_2 + \|\mathbf{Z}_{\text{\text{IR}}} - \mathbf{X}_{\text{\text{IR}}}\|^2_2\\
                    & + \xi_{\text{R}} \cdot tr(\mathbf{Z}^T_{\text{R}}\mathbf{L}_{\text{R}}\mathbf{Z}_{\text{R}}) + \xi_{\text{IR}} \cdot tr(\mathbf{Z}^T_{\text{IR}}\mathbf{L}_{\text{IR}}\mathbf{Z}_{\text{IR}})\\[5pt]
\textrm{where} \quad & \mathbf{L}_\text{R} = \mathbf{D}_\text{R} - \mathbf{A}_\text{R}, \mathbf{L}_\text{IR} = \mathbf{D}_\text{IR} - \mathbf{A}_\text{IR}\\[5pt]
\textrm{s.t.} \quad & \mathbf{A}_\text{R} + \mathbf{A}_\text{IR} = \mathbf{A}\\[5pt]
                    & \mathbf{A}_{\text{R}(i, j)}, \mathbf{A}_{\text{IR}(i, j)} \in [0,1].\\
\end{aligned}
\end{equation}
Here, $\mathbf{A}_{\text{R}(i,j)}$ and $\mathbf{A}_{\text{IR}(i,j)}$ measure the degree to which the node connection $(v_i,v_j)$ are relevant and irrelevant to the learning task, respectively. We further name this optimization as \textit{disentangled graph denoising problem}, and finally derive the following theorem:
\begin{theorem}
The proposed ES-GNN is equivalent to the solution of the disentangled graph denoising problem in Eq.~\eqref{eq:disen_gsd}.
\end{theorem}
\begin{proof}
Let $\mathbf{X}_\text{\text{R}} \in \mathbb{R}^{\frac{d}{2}}$ and $\mathbf{X}_{\text{\text{IR}}} \in \mathbb{R}^{\frac{d}{2}}$ be the results of mapping $\mathbf{X}$ into different channels in Eq.~\eqref{eq:channels_mp}, i.e., $\mathbf{X}_{\text{\text{R}}}=\mathbf{Z}_{\text{\text{R}}}^{(0)}$ and $\mathbf{X}_{\text{\text{IR}}}=\mathbf{Z}_{\text{\text{IR}}}^{(0)}$.  Hypothesis~\ref{hyp:our} motivates us to define $\mathbf{A}_{\text{R}(i,j)}$ and $\mathbf{A}_{\text{IR}(i,j)}$ as node similarity in two aspects. Combining above constraints, we have a linear system in case of $\mathbf{A}_{(i,j)}=1$:
\begin{equation} 
\begin{cases}
\mathbf{A}_{\text{R}(i,j)} + \mathbf{A}_{\text{IR}(i,j)} = 1\\ 
\mathbf{A}_{\text{R}(i,j)} - \mathbf{A}_{\text{IR}(i,j)} = \phi_\text{res}(\mathbf{Z}_{\text{R}[i,:]}, \mathbf{Z}_{\text{IR}[i,:]}, \mathbf{Z}_{\text{R}[j,:]}, \mathbf{Z}_{\text{IR}[j,:]})
\end{cases},
\end{equation}
where $\phi_\text{res}(\cdot)$ outputs the residual between $\mathbf{A}_{\text{R}(i,j)}$ and $\mathbf{A}_{\text{IR}(i,j)}$ considering both task-relevant and irrelevant node information, and can be formulated with our residual scoring mechanism in Eq.~\eqref{eq:rs_eq}. 
Solving above equations, we can express both $\mathbf{A}_\text{R}$ and $\mathbf{A}_\text{IR}$ in terms of $\mathbf{Z}_\text{R}$ and $\mathbf{Z}_\text{IR}$, i.e.,
\begin{equation}
\mathbf{A}_{\text{R}(i,j)} = \frac{1}{2}(1 + \alpha_{i, j}),\ 
\mathbf{A}_{\text{IR}(i,j)} = \frac{1}{2}(1 - \alpha_{i, j})\label{eq:theo_res}.
\end{equation}
where $\alpha_{i, j} = \phi_\text{res}(\mathbf{Z}_{\text{R}[i,:]}, \mathbf{Z}_{\text{IR}[i,:]}, \mathbf{Z}_{\text{R}[j,:]}, \mathbf{Z}_{\text{IR}[j,:]})$.
So far, the optimization problem in Eq.~\eqref{eq:disen_gsd} is only made up of variables $\mathbf{X}_\text{R}$, $\mathbf{X}_\text{IR}$, $\mathbf{Z}_\text{R}$, and $\mathbf{Z}_\text{IR}$. Directly solving it is still however not easy, as the mixing variables of $\mathbf{Z}_\text{R}$ and $\mathbf{Z}_\text{IR}$, and the introduced non-linear operator in $\phi_{\text{res}(\cdot)}$ result in a complicated differentiation process. 

Instead, we can approach this problem by decoupling the learning of $\mathbf{A}_\text{R},\mathbf{A}_\text{IR}$ from the optimization target, and employ an alternative learning between stages.
Suppose we have attained the task-relevant and irrelevant node features in the $k^{th}$ round, i.e., $\mathbf{Z}_\text{R}^{(k)}$ and $\mathbf{Z}_\text{IR}^{(k)}$. In the first stage, we can compute $\mathbf{A}_{\text{R}(i,j)}^{(k+1)}$ and $\mathbf{A}_{\text{IR}(i,j)}^{(k+1)}$ 
using $\{\mathbf{Z}_{\text{\text{R}}[i,:]}^{(k)}, \mathbf{Z}_{\text{\text{IR}}[i,:]}^{(k)}, \mathbf{Z}_{\text{\text{R}},[j,:]}^{(k)}, \mathbf{Z}_{\text{\text{IR}}[j,:]}^{(k)}\}$ with Eq.~\eqref{eq:theo_res}, which in fact turns out to be our ES-layer in Section~\ref{sec:es_layer}. 

In the second stage, injecting the computed values of $\mathbf{A}_{\text{R}(i,j)}^{(k+1)}$ and $\mathbf{A}_{\text{IR}(i,j)}^{(k+1)}$ relaxes the mixture of variables $\mathbf{Z}_\text{R}$ and $\mathbf{Z}_\text{IR}$, and the original optimization problem can then be disentangled into two independent targets (as all four penalized terms are positive):

\begin{align}
&\mathop{\arg \min }_{\mathbf{Z}_\text{\text{R}}^*} \|\mathbf{Z}_\text{\text{R}}^* - \mathbf{Z}_\text{\text{R}}^{(0)}\|^2_2 + \xi_{\text{R}} \cdot tr({\mathbf{Z}_\text{\text{R}}^*}^T\mathbf{L}_\text{\text{R}}^{(k)}\mathbf{Z}_\text{\text{R}}^*)\label{eq:rel_gsd}\\
&\mathop{\arg \min }_{\mathbf{Z}_{\text{\text{IR}}}^*} \|\mathbf{Z}_{\text{\text{IR}}}^* - \mathbf{Z}_{\text{\text{IR}}}^{(0)}\|^2_2 + \xi_{\text{IR}} \cdot tr({\mathbf{Z}_{\text{\text{IR}}}^*}^T\mathbf{L}_{\text{\text{\text{IR}}}}^{(k)}\mathbf{Z}_{\text{\text{IR}}}^*)\label{eq:irr_gsd}
\end{align}
where $\mathbf{L}_\text{R}^{(k)} = \mathbf{D}_\text{R}^{(k)} - \mathbf{A}_\text{R}^{(k)}$ and $\mathbf{L}_\text{IR}^{(k)} = \mathbf{D}_\text{IR}^{(k)} - \mathbf{A}_\text{IR}^{(k)}$ are fixed values. Lemma~\ref{lm:agg}, on the $\text{\text{R}}$ channel as an example, further shows that our aggregation layer, on the task-relevant and irrelevant topologies, in Section~\ref{sec:agg_ly} is approximately solving these two optimization problems in Eq.~\eqref{eq:rel_gsd} and Eq.~\eqref{eq:irr_gsd}.

Therefore, stacking ES- and aggregation layers iteratively is equivalent to the above alternative learning for solving the \textit{disentangled graph denoising problem} in Eq.~\eqref{eq:disen_gsd} with $\mathbf{X}_\text{R}=\mathbf{Z}_\text{R}^{(0)}$ and $\mathbf{X}_\text{IR}=\mathbf{Z}_\text{IR}^{(0)}$. 
Finally, given $\mathbf{Z}_{\text{\text{R}}}^{(K)}$ and $\mathbf{Z}_{\text{\text{IR}}}^{(K)}$, we minimize the prediction loss~$\mathcal{L}_{\text{pred}}$ and the Irrelevant Consistency Regularization~$\mathcal{L}_{\text{ICR}}$ in Eq.~\eqref{eq:final_op} with Adam~\cite{kingma2014adam} algorithm, which imposes concrete meanings on different channels, and simultaneously ensures the convergence of our described alternative learning.
\end{proof}

\begin{lemma}\label{lm:agg}
When adopting the normalized Laplacian matrix $\mathbf{L}_{\text{\text{R}}} = \mathbf{I} - \mathbf{D}_{\text{\text{R}}}^{-\frac{1}{2}} \mathbf{A}_{\text{\text{R}}} \mathbf{D}_{\text{\text{R}}}^{-\frac{1}{2}}$, the feature aggregation operator in Eq.~\eqref{eq:agg_eq} with channel $s=\text{\text{R}}$ can be regarded as solving Eq.~\eqref{eq:rel_gsd} using iterative gradient descent with stepsize $\beta=\frac{1}{2+2\xi_{\text{R}}}$ and $\xi_{\text{R}}=\frac{1}{\epsilon_{\text{\text{R}}}} - 1$.
\end{lemma}

\begin{proof}
We take iterative gradient descent with the stepsize $\beta$ to solve the denoising problem in Eq.~\eqref{eq:rel_gsd} (referred as $\mathcal{L}_\text{R}$) as follows:
{\small
\begin{align}
    \mathbf{Z}_\text{R}^{(k+1)} 
    &= \mathbf{Z}_\text{R}^{(k)} - \beta \cdot \dfrac{\partial \mathcal{L}_\text{R}}{\partial \mathbf{Z}_\text{R}^{*}}\at{\mathbf{Z}_\text{R}^{*}=\mathbf{Z}_\text{R}^{(k)}}\\
    &= 2\beta \mathbf{Z}_\text{R}^{(0)} + 2 \beta \xi_{\text{R}} (\mathbf{D}_\text{R}^{-\frac{1}{2}} \mathbf{A}_\text{R} \mathbf{D}_\text{R}^{-\frac{1}{2}})\mathbf{Z}_\text{R}^{(k)} + (1 - 2\beta - 2\beta \xi_{\text{R}}) \mathbf{Z}_\text{R}^{(k)}.
\end{align}
}
Setting $\beta$ as $\frac{1}{2 + 2\xi_{\text{R}}}$ gives us:
\begin{equation}
    \mathbf{Z}_\text{R}^{(k+1)} = \frac{1}{1 + \xi_{\text{R}}} \mathbf{Z}_\text{R}^{(0)} + \frac{\xi_{\text{R}}}{1 + \xi_{\text{R}}} (\mathbf{D}_\text{R}^{-\frac{1}{2}} \mathbf{A}_\text{R} \mathbf{D}_\text{R}^{-\frac{1}{2}}) \mathbf{Z}_\text{R}^{(k)},
\end{equation}
which is equivalent to Eq.~\eqref{eq:agg_eq} while choosing $\xi_{\text{R}}=\frac{1}{\epsilon_\text{R}} - 1$, i.e., 
\begin{equation}
    \mathbf{Z}_\text{R}^{(k+1)} = \epsilon_\text{R} \mathbf{Z}_\text{R}^{(0)} + (1-\epsilon_\text{R}) (\mathbf{D}_\text{R}^{-\frac{1}{2}} \mathbf{A}_\text{R} \mathbf{D}_\text{R}^{-\frac{1}{2}}) \mathbf{Z}_\text{R}^{(k)}.
\end{equation}
\end{proof}

As the possible classification-harmful similarity between nodes (hidden in $\mathbf{A}_\text{IR}$) can be excluded from $\mathbf{Z}_\text{R}$ and disentangled in $\mathbf{Z}_\text{IR}$ while optimizing Eq.~\eqref{eq:disen_gsd}, our ES-GNN presents a universal approach that theoretically guarantees good performance on different types of networks.

\subsection{Aligning Disentangled and Conventional Problems}
It is noted that Eq.~(\ref{eq:1st_gsd}) can be interpreted as a special case of Eq.~(\ref{eq:disen_gsd}) in specific graph scenarios. This situation arises in graphs where only edges connecting nodes with identical labels exist, indicating that the similarity between adjacent nodes should be beneficial in predicting their shared label. In such cases, task-irrelevant edges may not exist, as all connections inherently support the task. Consequently, in this scenario, the objective term involving $\mathbf{L}_{\text{IR}}$ in Eq.~(\ref{eq:disen_gsd}) becomes redundant, amounting to zero, and the need for disentangling $\mathbf{Z}_{\text{R}}$ and $\mathbf{Z}_{\text{IR}}$ is obviated. This leads to the simplification of Eq.~(\ref{eq:disen_gsd}) into the conventional graph denoising problem Eq.~(\ref{eq:1st_gsd}), conforming to the standard smoothness assumption across the entire graph.

In practice, completely smooth graphs devoid of edges linking nodes with different labels are rare. Nevertheless, for homophilic graphs where most edges connect nodes from the same class, such as in the citation network Cora with homophily ratio 0.81, Eq.~(\ref{eq:1st_gsd}) serves as a close approximation of Eq.~(\ref{eq:disen_gsd}). This approximation holds as the term involving $\mathbf{L}_{\text{IR}}$ in Eq.~(\ref{eq:disen_gsd}) becomes negligible and the inherently classification-harmful information in the graph is almost non-existent. Empirical evidence from Fig.~\ref{fig:feat_corr} in our study reinforces this understanding, showing that on homophilic graphs like Cora, the majority of informative content is retained in the task-relevant channel, underscoring the minimal presence of classification-harmful information.

\section{Experiments}\label{sec:exp}
We empirically evaluate our ES-GNN for node classification using both synthetic and real-world datasets in this section.

\subsection{Datasets and Experimental Setup}

\subsubsection{Real-World Datasets}
We consider 11 widely used benchmark datasets including both seven heterophilc graphs, i.e., Chameleon, Squirrel~\cite{rozemberczki2021multi}, Wisconsin, Cornell, Texas~\cite{Pei2020GeomGCNGG} (webpage networks), Actor~\cite{tang2009social} (co-occurrence network), and Twitch-DE~\cite{rozemberczki2021multi,lim2021large} (social network), as well as four homophilic graphs including Cora, Citeseer, Pubmed~\cite{Sen2008CollectiveCI} (citation networks), and Polblogs~\cite{adamic2005political,jin2020graph} (community network) with statistics shown in Table~\ref{tab:data_sta}. For Polblogs dataset, since node features are not provided, we use the rows of the adjacency matrix.

\begin{table}[t]
\centering
\caption{Statistics of real-world datasets.}
\label{tab:data_sta}
\label{tab:data_sta}
\setlength\tabcolsep{4pt}
\resizebox{0.48\textwidth}{!}{
\begin{tabular}{lccccccc}
\toprule
\textbf{Dataset}   &$|\mathcal{V}|$ &$|\mathcal{E}|$ &$F$ &$C$ & $\mathcal{H}$ & $\mathcal{H}_{\text{class}}$ & $\mathcal{H}_{\text{adjusted}}$ \\
\midrule
\textbf{Squirrel}  &5,201       &217,073       &2,089          &5         &0.22   &0.03   &0.01\\
\textbf{Chameleon} &2,227       &36,101       &2,325          &5         &0.23   &0.06   &0.03\\
\textbf{Wisconsin} &251       &499       &1,703          &5         &0.21   &0.09   &-0.17\\
\textbf{Cornell}   &183       &295       &1,703          &5         &0.30   &0.05   &-0.08\\
\textbf{Texas}     &183       &309       &1,703          &5         &0.11   &0.00   &-0.23\\
\textbf{Twtich-DE} &9,498       &153,138       &2,545          &2         &0.63   &0.14   &0.14\\
\textbf{Actor} &7,600   &33,544   &931   &5   &0.22   &0.01   &0.00\\
\midrule
\textbf{Cora}      &2,708       &5,429       &1,433          &7         &0.81   &0.77   &0.77\\
\textbf{Citeseer}  &3,327       &4,732       &3,703          &6         &0.74   &0.63   &0.67\\
\textbf{Pubmed}    &19,717       &44,338       &500       &3       &0.80       &0.66       &0.69\\
\textbf{Polblogs} &1,222   &16,714   &/   &2   &0.91   &0.81   &0.81\\
\bottomrule
\end{tabular}
}
\end{table}

\begin{table}[t]
\caption{Parameters for synthesizing graphs with varying homophily ratios.}
\label{tab:syn_param}
\centering
\setlength\tabcolsep{3pt}
\resizebox{0.48\textwidth}{!}{
\begin{tabular}{cccccccccccc}
\toprule
$\mathcal{H}_\text{syn}$ & \textbf{0.0} & \textbf{0.1} & \textbf{0.2} & \textbf{0.3} & \textbf{0.4} & \textbf{0.5} & \textbf{0.6} & \textbf{0.7} & \textbf{0.8} & \textbf{0.9} & \textbf{1.0}\\ 
\midrule
$P_\text{E}$       
&0.02           &0.06           &0.1           &0.2           &0.4           &0.4           &0.6           &0.7           &0.8           &0.9           &0.96           \\
$P_\text{I}$       
&0.72           &0.81           &0.6           &0.7           &0.9           &0.6           &0.6           &0.45           &0.3           &0.15           &0.045           \\
$\omega$          
&0.1           &0.084           &0.1           &0.075           &0.05           &0.062           &0.05           &0.05           &0.05           &0.05           &0.051           \\ 
\bottomrule
\end{tabular}}
\end{table}

\begin{figure}[t]
\centering
\includegraphics[clip, trim=0cm 15cm 0cm 0cm,width=0.48\textwidth]{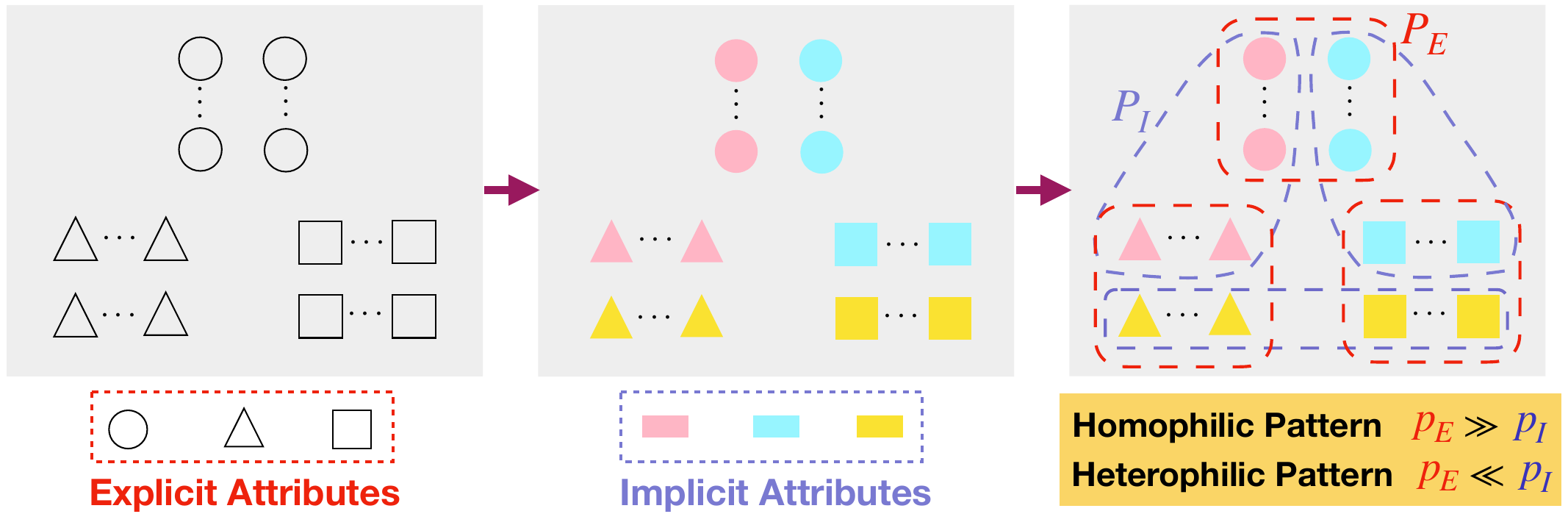}
\caption{Synthetic graphs with varying levels of homophily. Node shape and color refer to the explicit and implicit attributes, respectively. Nodes sharing the same shape (or color) are connected with a probability of $P_\text{E}$ (or $P_\text{I}$) and are classified into three categories only based on their different shapes. In this context, ``shape'' attributes represent task-relevant features, whereas ``color'' attributes denote task-irrelevant ones. It can be intuitively observed that adequate disentanglement of these attributes is crucial for classification tasks; otherwise, model prediction will inevitably suffer, as misled by the task-irrelevant ``color'' information.}
\label{fig:gen_syn}
\end{figure}

\begin{table*}[t]
\caption{Node classification accuracies (\%) over 100 runs. Error Reduction gives the average improvement of ES-GNN upon baselines w/o Basic GNNs.}\label{tab:real_nc}
\setlength\tabcolsep{4pt}
\resizebox{0.98\textwidth}{!}{
\begin{tabular}{lccccccc|cccc}
\toprule
\multirow{2}{*}{\textbf{Datasets}} & \multicolumn{7}{c|}{\textbf{Heterophilic Graphs}}                                          & \multicolumn{4}{c}{\textbf{Homophilic Graphs}}                                            \\ 
\cmidrule{2-12}
& \textbf{Squirrel}    & \textbf{Chameleon}   & \textbf{Wisconsin} & \textbf{Cornell} & \textbf{Texas}  & \textbf{Twitch-DE}      &{\textbf{Actor}}    & \textbf{Cora}        & \textbf{Citeseer}    & \textbf{Pubmed}      & \textbf{Polblogs}    \\ 
\midrule
GCN~\cite{kipf2017semi}     
&55.2{$\pm$1.5}    &67.6{$\pm$2.0}    &59.5{$\pm$3.6}    &52.8{$\pm$6.0}     &61.7{$\pm$3.7}        &74.0{$\pm$1.2}  &{31.2{$\pm$1.3}}  
&79.7{$\pm$1.2}    &69.5{$\pm$1.7}      &78.7{$\pm$1.6}     &89.4{$\pm$0.9}           \\
SGC~\cite{wu2019simplifying}     
&50.7{$\pm$1.3}    &61.9{$\pm$2.6}    &53.7{$\pm$3.9}    &51.2{$\pm$0.9}     &51.4{$\pm$2.2}        &73.9{$\pm$1.3}    &{30.9{$\pm$0.6}}   
&79.1{$\pm$1.0}    &69.9{$\pm$2.0}    &76.6{$\pm$1.3}    &89.0{$\pm$1.5}        \\
GAT~\cite{velickovic2018graph}     
&54.8{$\pm$2.2}    &67.3{$\pm$2.2}    &57.9{$\pm$4.5}    &50.4{$\pm$5.9}     &55.4{$\pm$5.9}        &73.7{$\pm$1.3}    &{30.5{$\pm$1.2}} 
&82.0{$\pm$1.1}    &69.9{$\pm$1.7}    &78.6{$\pm$2.0}      &87.4{$\pm$1.1}              \\ 
\midrule
NeuralSparse~\cite{zheng2020robust}
&40.0{$\pm$1.6}
&60.5{$\pm$2.0}
&70.8{$\pm$3.4}
&64.1{$\pm$5.5}
&66.4{$\pm$5.7}
&71.3{$\pm$1.3}
&35.5{$\pm$1.1}
&78.5{$\pm$1.4}
&69.7{$\pm$1.8}
&79.1{$\pm$1.2}
&89.3{$\pm$0.9}
\\
GCN-LPA~\cite{wang2020unifying}
&54.2{$\pm$1.1}
&63.4{$\pm$1.9}
&63.3{$\pm$3.7}
&65.6{$\pm$7.3}
&61.2{$\pm$7.6}
&74.0{$\pm$1.2}
&37.8{$\pm$0.9}
&80.4{$\pm$1.5}
&69.7{$\pm$1.7}
&79.7{$\pm$1.3}
&\textbf{89.7{$\pm$0.8}}\\
\midrule
DisenGCN~\cite{ma2019disentangled}
&42.4{$\pm$1.6}
&58.4{$\pm$2.3}
&78.1{$\pm$4.0}
&77.4{$\pm$4.4}
&71.3{$\pm$5.7}
&73.5{$\pm$1.7}
&36.7{$\pm$1.2}
&81.5{$\pm$1.3}
&69.2{$\pm$1.7}
&\underline{80.0{$\pm$1.6}}
&89.5{$\pm$0.9}
\\
FactorGCN~\cite{yang2020factorizable}     
&56.6{$\pm$2.4}        &69.8{$\pm$2.0}        &64.2{$\pm$4.8}        &50.6{$\pm$1.8}        &69.5{$\pm$6.5}        &73.1{$\pm$1.4}        &{29.0{$\pm$1.4}}            &75.2{$\pm$1.6}        &61.6{$\pm$2.0}        &72.9{$\pm$2.3}        &87.9{$\pm$1.7}\\    
VEPM~\cite{he2022variational}
&50.3{$\pm$1.7}
&67.3{$\pm$2.1}
&55.6{$\pm$4.9}
&51.2{$\pm$7.0}
&55.8{$\pm$4.3}
&73.3{$\pm$1.2}
&29.3{$\pm$1.1}
&82.2{$\pm$1.2}
&69.1{$\pm$1.9}
&78.8{$\pm$1.6}
&89.5{$\pm$0.9}\\
DisGNN~\cite{zhao2022exploring}
&55.1{$\pm$4.8}
&68.2{$\pm$1.9}
&54.6{$\pm$5.4}
&52.0{$\pm$5.7}
&60.6{$\pm$3.9}
&69.2{$\pm$0.8}
&30.2{$\pm$1.3}
&78.2{$\pm$1.4}
&66.2{$\pm$2.2}
&77.6{$\pm$1.7}
&\underline{89.6{$\pm$0.9}}
\\
\midrule
GEN~\cite{wang2021graph}     
&36.0{$\pm$4.0}    &57.6{$\pm$3.1}    &83.3{$\pm$3.6}    &81.0{$\pm$3.9}     &78.3{$\pm$8.0}        &74.1{$\pm$1.4}    &{37.3{$\pm$1.4}}
&79.8{$\pm$1.3}    &69.7{$\pm$1.6}    &78.9{$\pm$1.7}    &\underline{89.6{$\pm$1.4}}        \\
WRGAT~\cite{suresh2021breaking}     
&39.6{$\pm$1.4}    &57.7{$\pm$1.6}    &82.9{$\pm$4.5}    &79.2{$\pm$3.5}     &80.5{$\pm$6.1}        &70.0{$\pm$1.3}    &{\underline{38.6{$\pm$1.1}}}  
&71.7{$\pm$1.5}    &64.1{$\pm$1.9}     &73.3{$\pm$2.1}    &88.2{$\pm$1.2}        \\
H2GCN~\cite{zhu2020beyond}     
&45.1{$\pm$1.9}    &62.9{$\pm$1.9}    &82.6{$\pm$4.0}    &79.6{$\pm$4.9}     &79.8{$\pm$7.3}        &73.1{$\pm$1.5}    &{38.4{$\pm$1.0}} 
&81.4{$\pm$1.4}     &68.7{$\pm$2.0}    &78.0{$\pm$2.0}    &89.0{$\pm$1.0}         \\
FAGCN~\cite{fagcn2021}
&50.4{$\pm$2.6}    &68.9{$\pm$1.8}    &82.3{$\pm$4.4}    &79.4{$\pm$5.5}     &80.3{$\pm$5.5}        &74.1{$\pm$1.4}    &{37.9{$\pm$1.0}}  
&\underline{82.6{$\pm$1.3}}     &70.3{$\pm$1.6}     &\underline{80.0{$\pm$1.7}}    &89.3{$\pm$1.1}         \\
GPR-GNN~\cite{chien2021adaptive}     
&54.1{$\pm$1.6}    &69.6{$\pm$1.7}    &82.7{$\pm$4.1}    &79.9{$\pm$5.3}     &81.7{$\pm$4.9}        &74.0{$\pm$1.6}    &{38.0{$\pm$1.1}}  
&81.5{$\pm$1.5}    &69.6{$\pm$1.7}    &79.8{$\pm$1.3}    &89.5{$\pm$0.8}        \\ 
GloGNN++~\cite{li2022finding}
&\underline{63.3{$\pm$1.2}}
&71.4{$\pm$2.0}
&\underline{84.9{$\pm$4.2}}
&82.0{$\pm$3.5}
&81.4{$\pm$5.6}
&72.8{$\pm$1.1}
&38.2{$\pm$1.2}
&80.9{$\pm$1.4}
&\underline{70.5{$\pm$1.9}}
&76.8{$\pm$2.1}
&\underline{89.6{$\pm$0.8}}
\\
ACM-GCN~\cite{luan2022revisiting}
&\textbf{67.0{$\pm$1.3}}
&\textbf{75.3{$\pm$2.2}}
&84.3{$\pm$4.5}
&\underline{82.1{$\pm$4.9}}
&\underline{82.2{$\pm$5.9}}
&\underline{74.2{$\pm$0.9}}
&36.6{$\pm$1.0}
&81.3{$\pm$1.0}
&69.4{$\pm$1.7}
&79.5{$\pm$1.4}
&\underline{89.6{$\pm$0.9}}
\\
GOAL~\cite{zheng2023finding}
&57.9{$\pm$0.9}
&71.3{$\pm$2.0}
&70.5{$\pm$5.1}
&54.9{$\pm$6.6}
&72.0$\pm$7.4
&68.5$\pm$1.5
&36.3{$\pm$1.0}
&80.6{$\pm$1.4}
&69.7$\pm$2.0
&78.7$\pm$1.3
&88.7{$\pm$1.6}
\\
\midrule
ES-GNN (ours)    
&62.4{$\pm$1.4}    &\underline{72.3{$\pm$2.1}}    &\textbf{85.3{$\pm$4.6}}    &\textbf{82.2{$\pm$4.0}}     &\textbf{82.3{$\pm$5.7}}        &\textbf{74.7{$\pm$1.1}}  &{\textbf{38.9{$\pm$0.8}}}  
&\textbf{83.0{$\pm$1.1}}    &\textbf{70.7{$\pm$1.7}}     &\textbf{80.7{$\pm$1.4}}    &\textbf{89.7{$\pm$0.9}}         \\ 
Error Reduction 
&\textbf{11.5\%}
&\textbf{6.4\%}
&\textbf{11.0\%}
&\textbf{11.7\%}
&\textbf{9.4\%}
&\textbf{2.2\%}
&\textbf{3.2\%}
&\textbf{3.3\%}
&\textbf{2.3\%}
&\textbf{2.6\%}
&\textbf{0.5\%}
\\
\bottomrule
\end{tabular}}
\end{table*}

\subsubsection{Synthetic Data}\label{sec:syn_data_gen}
To investigate the behavior of GNNs on graphs with arbitrary levels of homophily and heterophily, we consider the contextual stochastic block model (CSBM)~\cite{deshpande2018contextual,palowitch2022synthetic} to construct synthetic graphs with our Hypothesis~\ref{hyp:our} as guide. The central idea is to define links among nodes under two conditions independently, of which only one is correlated with the classification task. We consider 1,200 nodes, 3 equal-size classes, and 500 node features made up of both explicit and implicit attributes. The explicit attributes determine the label assignment, while implicit ones model dependency across different classes. Fig.~\ref{fig:gen_syn} further illustrates their allocation to nodes with ``shape'' and ``color'' as an example. Notably, all these attributes in six types (three explicit and three implicit ones) are randomly sampled from different Gaussian distributions, each pair of them are combined via element-wise addition to attain the final node features. For instance, the features of a node (from class-$i$) with explicit attribute-$i$ and implicit attribute-$j$ are defined as the addition of two random vectors respectively sampled from $\mathcal{N}(\boldsymbol{\mu}_{\text{E},i},\boldsymbol{\sigma}_{\text{E},i})$ and $\mathcal{N}(\boldsymbol{\mu}_{\text{I},j},\boldsymbol{\sigma}_{\text{I},j})$, where $\boldsymbol{\mu}_{\text{E},i},\boldsymbol{\mu}_{\text{I},j} \in \mathbb{R}^{F_\text{syn}}$ are means, $\boldsymbol{\sigma}_{\text{E},i},\boldsymbol{\sigma}_{\text{I},j} \in \mathbb{R}^{F_\text{syn} \times F_\text{syn}}$ are the associated covariance matrixes, and $F_\text{syn}=500$ is the feature dimensions. Then, we connect nodes with probability $P_\text{E}$ if they are from the same class (the task-relevant condition), with probability $P_\text{I}$ if they share different labels but
posses implicit attributes from the same distribution (the task-irrelevant condition). For all other cases, we connect nodes with probability $q$ in a small value, $1\mathrm{e}{-5}$ in this work for ensuring a connected graph. Since no class-imbalance problem exists here, the homophily ratios of our generated graphs are measured using index $\mathcal{H}$. Intuitively, we could anticipate heterophilic connecting pattern when setting $P_\text{E} \ll P_\text{I}$, and strong homophily otherwise. Quantitatively, the relationship between the homophily ratio $\mathcal{H}_\text{syn}$ and parameters $P_\text{E}$, $P_\text{I}$ can be derived with the simple knowledge on combinatorics and statistics while omitting the small value of $q$: 
$\mathcal{H}_\text{syn}(P_\text{E}, P_\text{I}) = \frac{3 (N_\text{syn}-3) }{3 (N_\text{syn}-3)  + 2 N_\text{syn} \frac{P_\text{I}}{P_\text{E}}}$,
with $N_\text{syn}$ being the total number of nodes. Clearly, we have $\mathcal{H}_\text{syn} \to 0$ while $P_\text{I} \gg P_\text{E}$, and $\mathcal{H}_\text{syn} \to 1$ while $P_\text{I} \ll P_\text{E}$. To avoid possible computational overhead, we also need to control the average node degree of our synthetic graphs. Similarly, we can approximately derive it as the function of $P_\text{I}$ and $P_\text{E}$:
$\mathcal{T}(P_\text{E}, P_\text{I}) = \frac{N_\text{syn}-3}{3}P_\text{E} + \frac{4 N_\text{syn}}{9}P_\text{I}$.
Give this, 
we have that $\mathcal{H}_\text{syn}(\cdot)$ is a function of the fraction between $P_\text{E}$ and $P_\text{I}$ with fixed $n$, and $\mathcal{T}(\cdot)$ is linearly correlated with $P_\text{E}$ and $P_\text{I}$. As such, given fixed $P_\text{E}$ and $P_\text{I}$ attaining certain $\mathcal{H}_{\text{syn}}$, we can almost attain the average node degree in any values with a scaling parameter $\omega$, i.e., average degree $=\omega \cdot \mathcal{T}(P_\text{E}, P_\text{I})=\mathcal{T}(\omega \cdot P_\text{E}, \omega \cdot P_\text{I})$ without changing $\mathcal{H}_\text{syn}$. 
In this work, we tune all these parameters such that the average degree is around 20, and list the experimented values in Table~\ref{tab:syn_param}.

\subsubsection{Data Splitting}
For heterophilic graphs and our synthetic graphs, we divides each dataset into 60\%/20\%/20\% corresponding to training/validation/testing to follow~\cite{Pei2020GeomGCNGG,zhu2020beyond,chien2021adaptive}. For homophilic graphs, we adopt the popular sparse splitting~\cite{kipf2017semi,velickovic2018graph,wu2019simplifying}, i.e., 20 nodes per class, 500 nodes, and 1,000 nodes to train, validate, and test models. For each dataset, 10 random splits are created for evaluation.

\subsubsection{Baselines}
We compare our ES-GNN with 17 baseline models, categorized into four groups: (1) Basic GNNs: GCN~\cite{kipf2017semi}, SGC~\cite{wu2019simplifying}, and GAT~\cite{velickovic2018graph}; (2) GNNs prioritizing task-relevance:, NeuralSparse~\cite{zheng2020robust}, and GCN-LPA~\cite{wang2020unifying}; (3) GNNs disentangling graphs: DisenGCN~\cite{ma2019disentangled}, FactorGCN~\cite{yang2020factorizable}, DisGNN~\cite{zhao2022exploring}, and VEPM~\cite{he2022variational}; (4) GNNs tailored for heterophily: GEN~\cite{wang2021graph}, WRGAT~\cite{suresh2021breaking}, H2GCN~\cite{zhu2020beyond}, FAGCN~\cite{fagcn2021}, GPR-GNN~\cite{chien2021adaptive}, GloGNN++~\cite{li2022finding}, ACM-GCN~\cite{luan2022revisiting}, and GOAL~\cite{zheng2023finding}.

\subsubsection{Implementation Details}
For all the baselines and our model, we set $d=64$ as the number of hidden states for fair comparison, and tune the hyper-parameters on the validation split of each dataset using Optuna~\cite{akiba2019optuna} for 200 trials. With the best hyper-parameters, we train models in 1,000 epochs using the early-stopping strategy with a patience of 100 epochs. We then report the models' average performance across 10 runs on the test set for each of the 10 random splits, leading to a total of 100 runs. For reproducibility, we provide the searching space of our hyper-parameters: learning rate $\sim \left[1\mathrm{e}{-2}, 1\mathrm{e}{-1}\right]$, weight decay $\sim \left[1\mathrm{e}{-6}, 1\mathrm{e}{-3}\right]$, dropout $\sim \{0,0.1,...,0.8\}$ with step $0.1$, the number of layers $K \sim \{1,2,...,8\}$ with step 1, scaling parameters $\epsilon_\text{R},\epsilon_\text{IR} \sim  \left[5\mathrm{e}{-2}, 0.5\right]$, and irrelevant consistency coefficient $\lambda_\text{ICR} \sim \left[0,1\right]$ for Cora, Citeseer, Pubmed, and Twitch-DE, $\left[5\mathrm{e}{-8},5\mathrm{e}{-6}\right]$ for Chameleon, Wisconsin, Cornell, and Texas, $\left[5\mathrm{e}{-5},5\mathrm{e}{-3}\right]$ for Squirrel, and $\left[5\mathrm{e}{-3},5\mathrm{e}{-2}\right]$ for Actor. Our implementation can be found at \url{https://github.com/jingweio/ES-GNN}.

\subsection{Results on Real-World Graphs}

Table~\ref{tab:real_nc} summaries node classification accuracies on real-world datasets over 100 runs with multiple random splits and various model initializations. Generally, our ES-GNN outperforms competitors on most datasets, except for ranking third on Squirrel and second on Chameleon against a wide array of baseline models. In particular, compared to both GNNs specializing in task-relevance, graph disentanglement, and heterophily, our method achieves an average improvements of 11.5\%, 6.4\%, 11.0\%, 11.7\%, and 9.4\% on heterophilic graphs like Squirrel, Chameleon, Wisconsin, Cornell, and Texas, respectively. On the Twitch-DE and Actor datasets, ES-GNN leads by a smaller margin, with an average increase of 2.7\%. In strong homophilic settings, where the majority of edges are intra-class links -- essential for node classification -- ES-GNN not only capitalizes on these connections but also effectively mitigates the potential noise propagation caused by a small number of inter-class edges. This capability ensures that ES-GNN remains competitive, demonstrating an average performance advantage of 2.2\% on the Cora, Citeseer, Pubmed, and Polblogs datasets. In this homophily context, we will further demonstrate the remarkable robustness of ES-GNN in case of perturbation or noisy links in Section~\ref{sec:robust}.

\begin{figure*}[t]
    \centering
    \subfloat[Chameleon]{\includegraphics[width=0.19\textwidth]{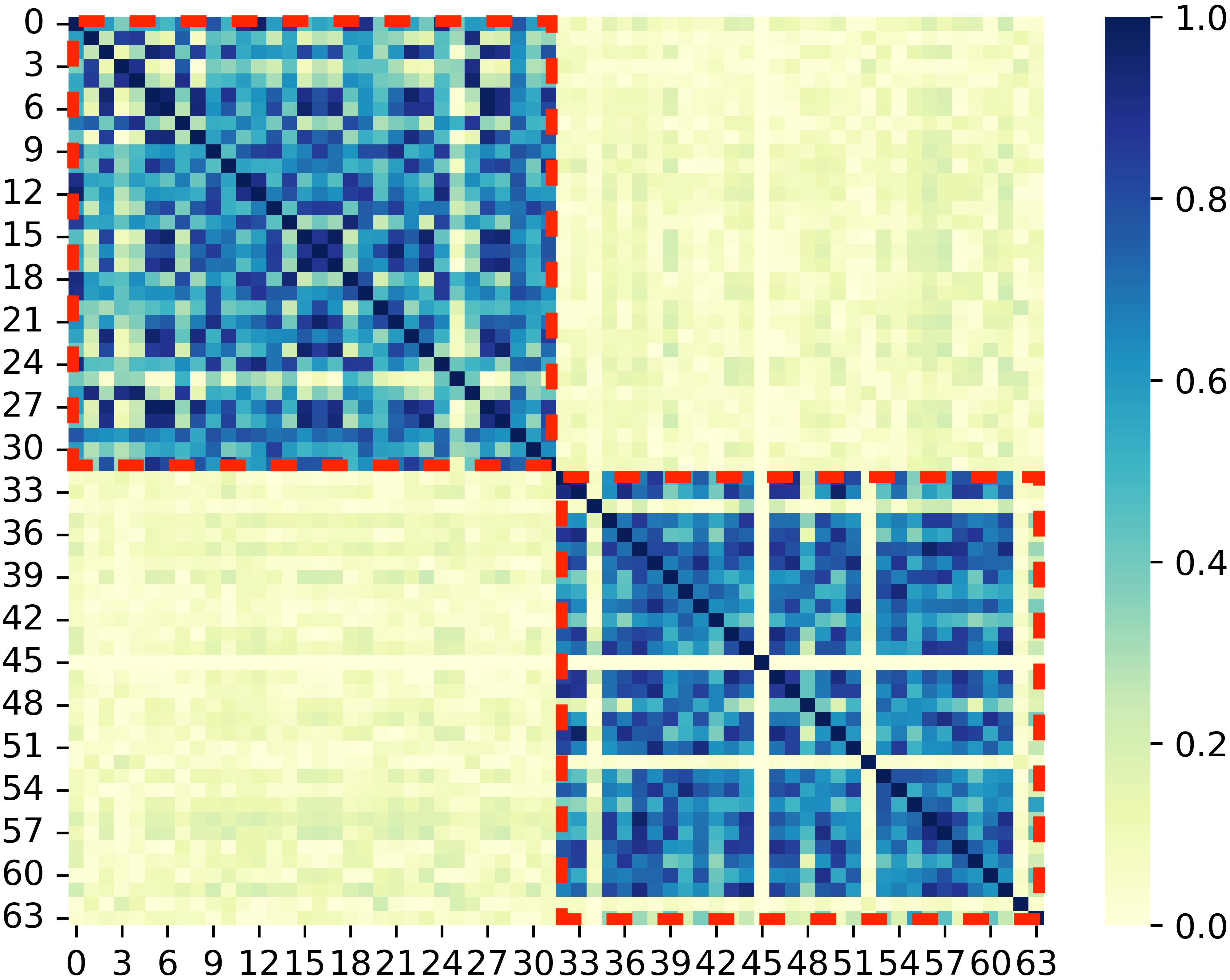}\label{fig:disen_cham}}
    \hfil
    \subfloat[Cora]{\includegraphics[width=0.19\textwidth]{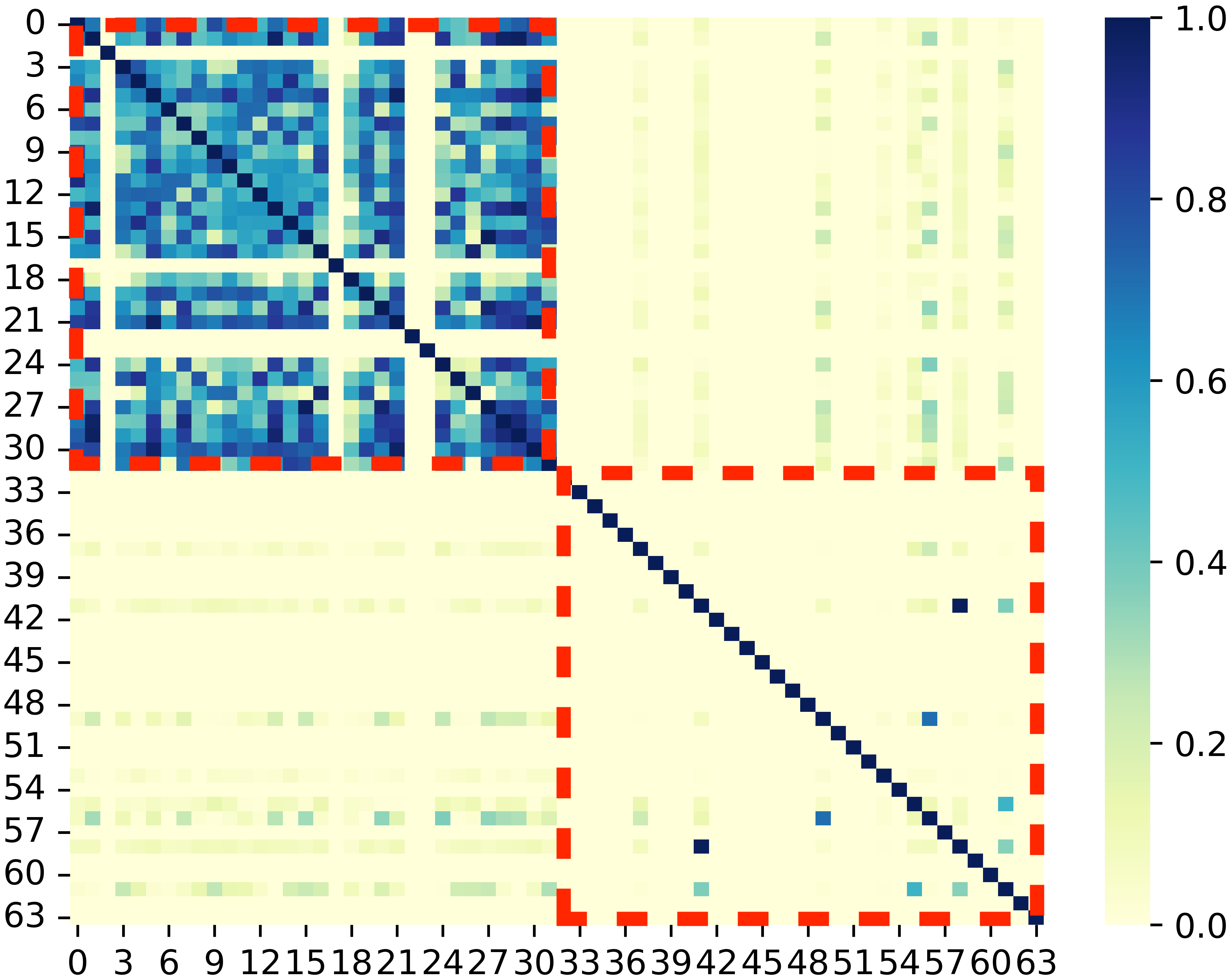}\label{fig:disen_cora}}
    \hfil
    \subfloat[$\mathcal{H}_\text{syn}=0.1$]{\includegraphics[width=0.19\textwidth]{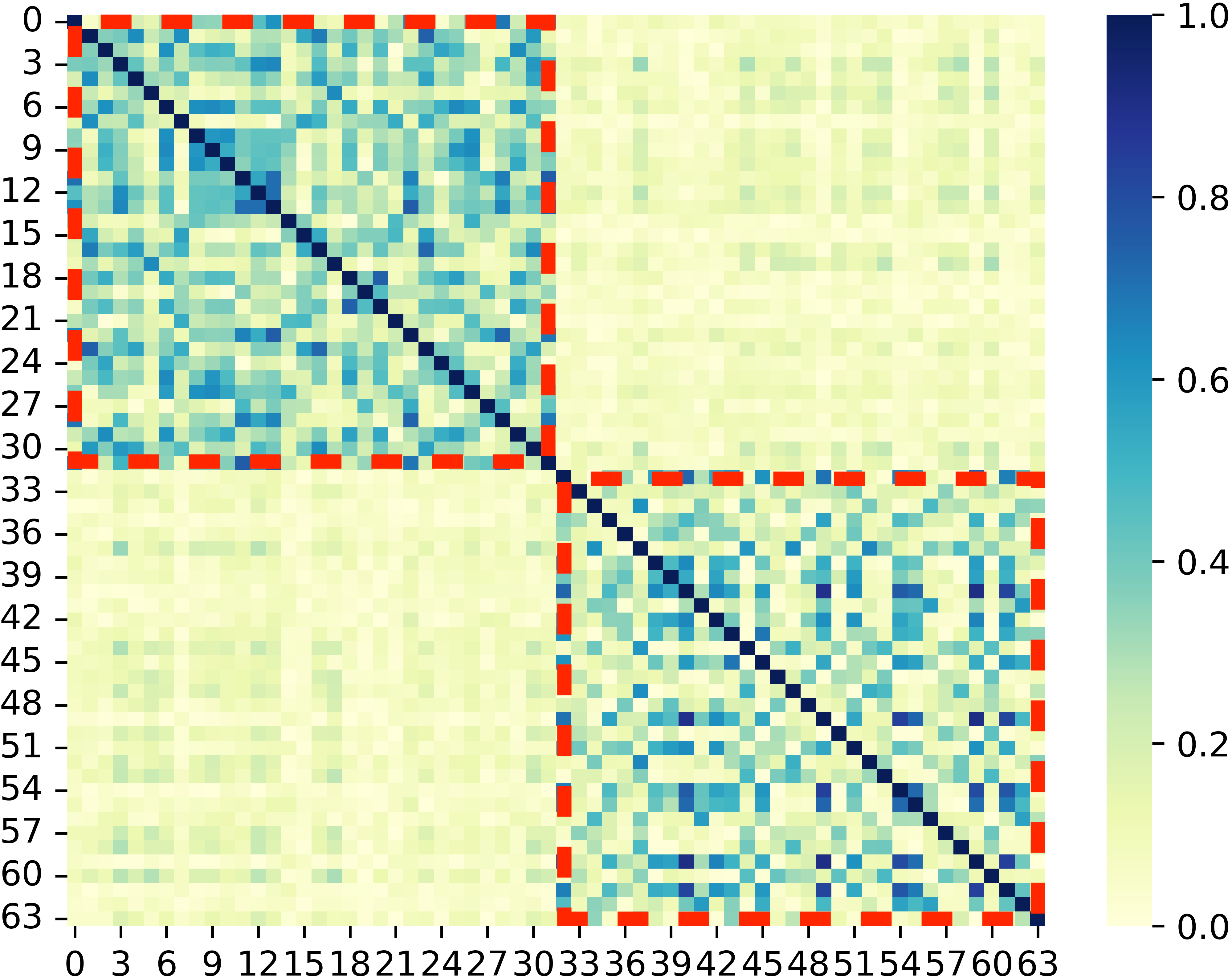}\label{fig:disen_syn_0.1}}
    \hfil
    \subfloat[$\mathcal{H}_\text{syn}=0.5$]{\includegraphics[width=0.19\textwidth]{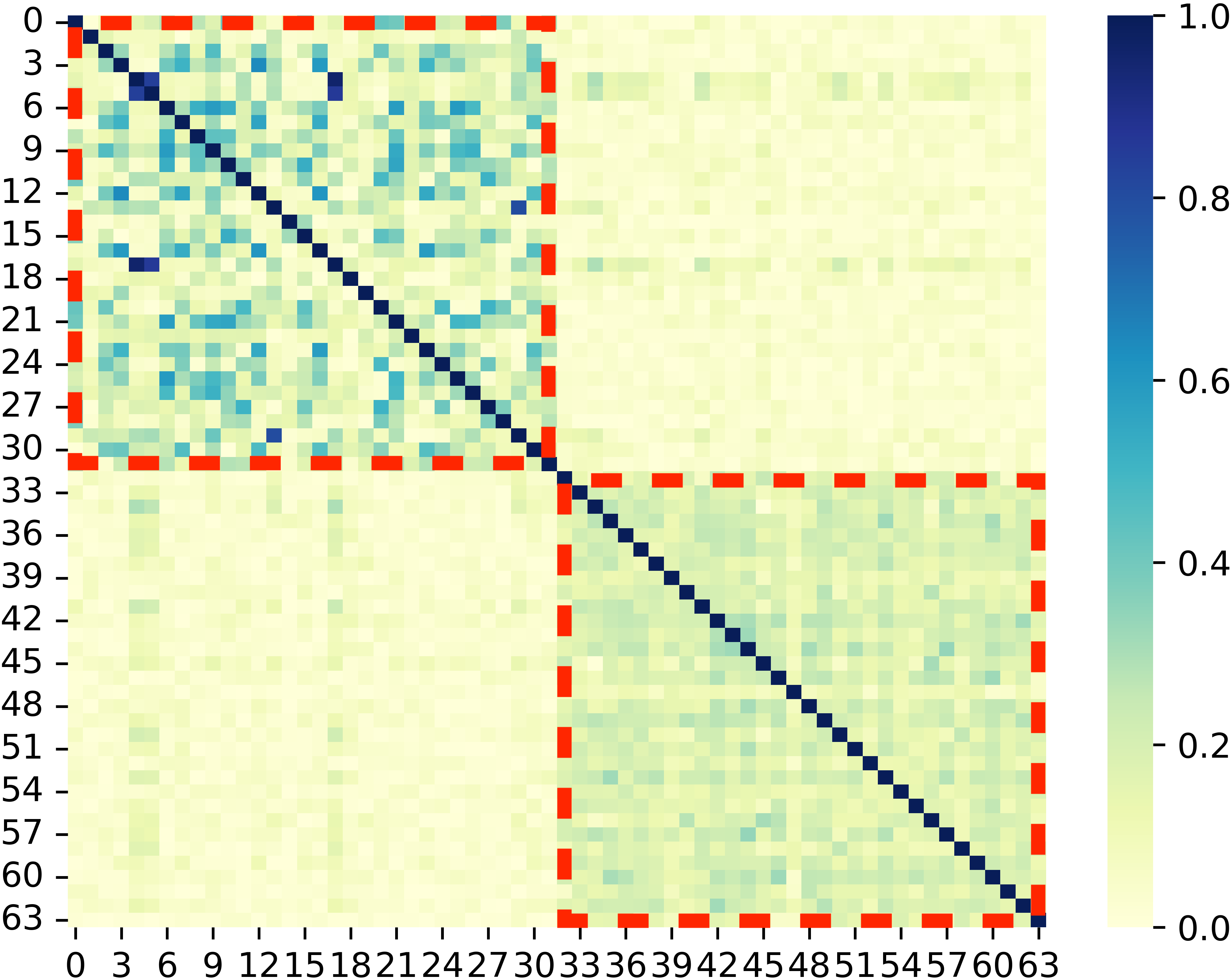}\label{fig:disen_syn_0.5}}
    \hfil
    \subfloat[$\mathcal{H}_\text{syn}=0.9$]{\includegraphics[width=0.19\textwidth]{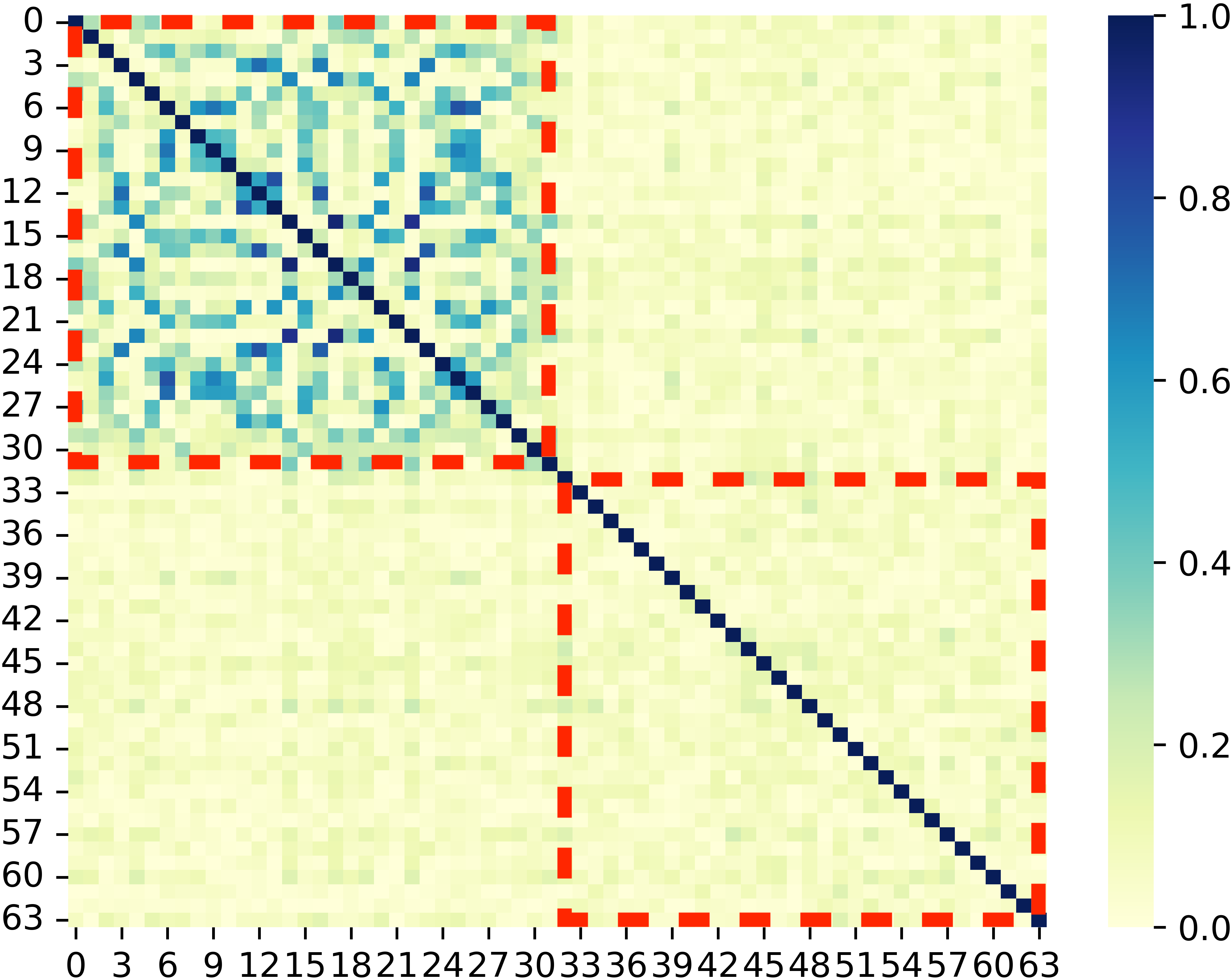}\label{fig:disen_syn_0.9}}
    \caption{Feature correlation analysis. Two distinct patterns (task-relevant and task-irrelevant topologies) can be learned on Chameleon with $\mathcal{H}=0.23$, while almost all information is retained in the task-relevant channel (0-31) on Cora with $\mathcal{H}=0.81$. 
    On synthetic graphs in (c), (d), and (e), block-wise pattern in the task-irrelevant channel (32-63) is gradually attenuated with the incremental homophily ratios across $0.1$, $0.5$, and $0.9$.
    ES-GNN presents one general framework which can be adaptive for both heterophilic and homophilic graphs.}\label{fig:feat_corr}
\end{figure*}

\begin{figure}[!t]
    \centering
    \includegraphics[width=0.38\textwidth]{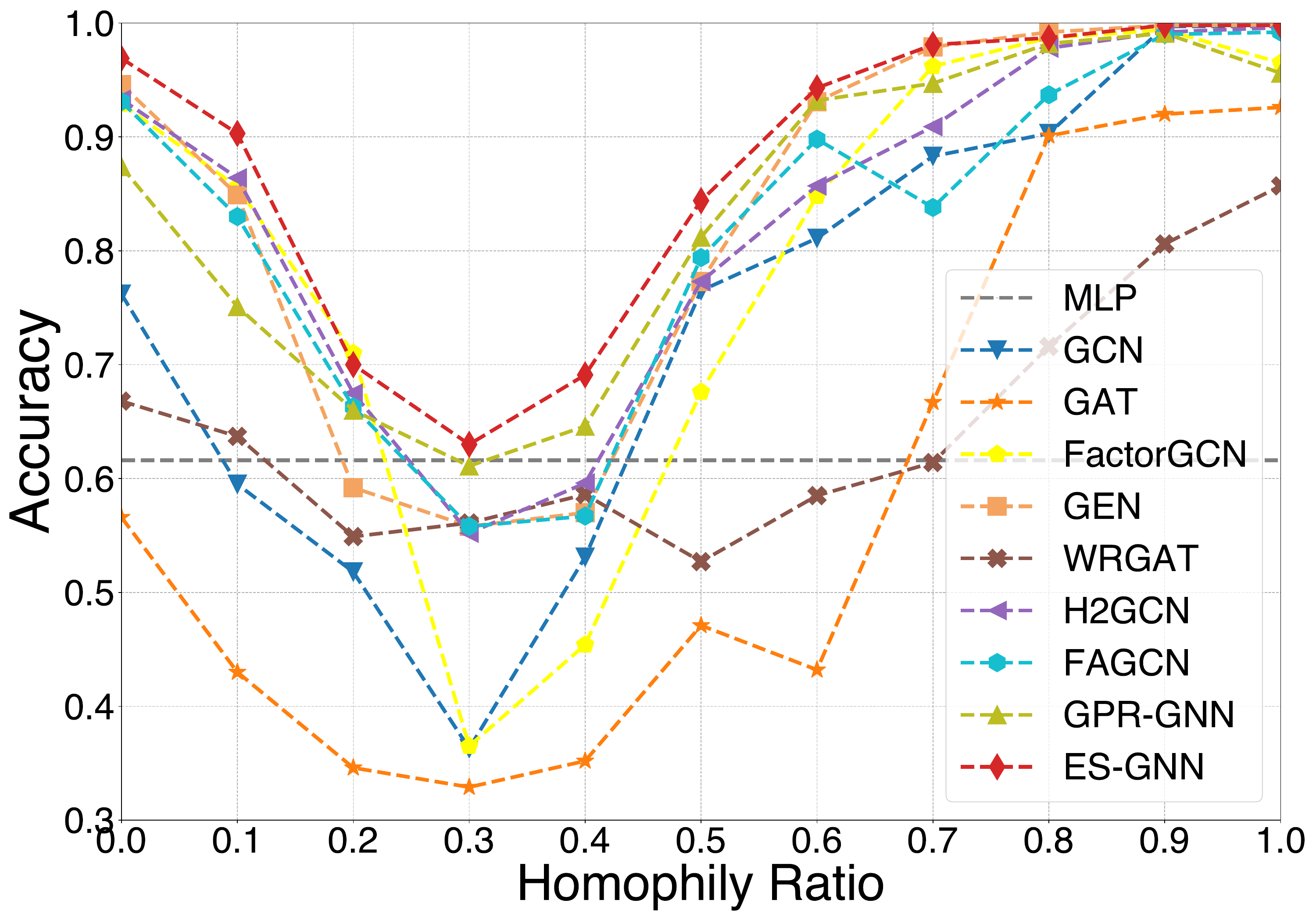}
    \caption{Results of different models on synthetic graphs with varied homophily ratios, where ES-GNN constantly outperform all the baselines.}
    \label{fig:syn_nc}
\end{figure}

\subsection{Results on Synthetic Graphs}\label{sec:re_syn_data}
We examine the learning ability of various models on graphs across the homophily or heterophily spectrum. 
From Fig.~\ref{fig:syn_nc}, we have the following observations:
\textbf{1)}~Looking through the overall trend, we obtain a ``U'' pattern on graphs from the lowest to the highest homophily ratios. That suggests GNNs' prediction performance is not monotonically correlated with graph homophily levels in a strict manner. When it comes to the extreme heterophilic scenario, GNNs tend to alternate node features completely between different classes, thereby still making nodes distinguishable w.r.t. their labels, which coincides with the findings in~\cite{Ma2021IsHA}.
\textbf{2)}~Despite the attention mechanism for adaptively utilizing relevant neighborhood information, GAT turns out to be the least robust method to arbitrary graphs. The entangled information in the mixed assortativity and disassortativity provides weak supervision signals for learning the attention weights. FactorGCN employs a graph factorization to disentangle different graph aspects but still adopts all of them for prediction without judgement, thereby performing poorly especially on the tough cases of $\mathcal{H}_\text{syn}=0.3, 0.4, \text{ and } 0.5$. 
\textbf{3)}~Both FAGCN and GPR-GNN model the dissimilarity between nearby nodes to go beyond the smoothness assumption in conventional GNNs, and display some superiority under heterophily. However, the correlation between graph edges and classification tasks is not explicitly defined and emphasized in their designs. In other words, the classification-harmful information still could be preserved in their node dissimilarity. Experimental results also show that these methods are constantly beaten by our disentangled approach.
\textbf{4)}~The proposed ES-GNN consistently outperforms, or matches, others across different graphs with different homophily levels, especially in the hardest case with $\mathcal{H}_\text{syn}=0.3$ where some baselines even perform worse than MLP. This is mainly because our ES-GNN is able to distinguish between task-relevant and irrelevant graph links, and makes prediction with the most correlated features only. We further provide detailed analyses in the following sections.

\begin{table}[t]
\caption{Edge Analysis of our ES-GNN on synthetic graphs with various homophily ratios. ``Removed Het.'' gives the percentage (\%) of heterophilic (inter-class) node connections excluded from the task-relevant topology and disentangled in the task-irrelevant topology. The last two rows list the corresponding node classification accuracies (\%) of ES-GNN and its variant while ablating ES-layer.}\label{tab:syn_edge_analysis}
\centering
\setlength\tabcolsep{4pt}
\resizebox{0.48\textwidth}{!}{
\begin{tabular}{lccccccccc|c}
\toprule
\multicolumn{1}{l|}{$\mathcal{H}_\text{syn}$}   &\textbf{0.1}    &\textbf{0.2}    &\textbf{0.3}    &\textbf{0.4}    &\textbf{0.5}    &\textbf{0.6}    &\textbf{0.7}    &\textbf{0.8}    &\textbf{0.9}    &\textbf{Avg.}    \\ \midrule
\multicolumn{1}{l|}{Removed Het.}  &41.9     &53.2     &60.8     &70.4     &74.2     &80.7     &86.7     &87.8     &89.9     &\textbf{71.7}     \\
\midrule
\multicolumn{1}{l|}{ES-GNN}     &90.0     &69.6     &62.1     &69.6     &85.4     &93.8     &98.3     &99.2     &100.0     &\textbf{85.3}     \\
\multicolumn{1}{l|}{ES-GNN w/o ES}     &84.6     &57.9     &53.3     &53.8     &74.2     &81.7     &86.3     &90.4     &96.7     &\textbf{75.4}     \\
\bottomrule
\end{tabular}
}
\end{table}

\subsection{Correlation Analysis}\label{sec:corr_analy}
To better understand our proposed method, we investigate the disentangled features on Chameleon, Cora, and three synthetic graphs as typical examples in Fig.~\ref{fig:feat_corr}. Clearly, on the strong heterophilic graph Chameleon with $\mathcal{H}=0.23$, correlation analysis of learned latent features displays two clear block-wise patterns, each of which represents task-relevant or task-irrelevant aspect respectively. In contrast, on the citation network Cora with $\mathcal{H}=0.81$, the node connections are in line with the classification task, since scientific papers mostly cite or are cited by others in the same research topic. Thus, most information will be retained in the task-relevant topology, while very minor information could be disentangled in the task-irrelevant topology (see Fig.~\ref{fig:disen_cora}).  
On the other hand, the results on synthetic graphs from Fig.~\ref{fig:disen_syn_0.1} to \ref{fig:disen_syn_0.9} display an attenuating trend on the second block-wise pattern with the incremental homophily ratios across $0.1, 0.5, \text{ and } 0.9$. This correlation analysis empirically verifies that our ES-GNN successfully disentangles the task-relevant and irrelevant features, and also demonstrates its universal adaptivity on different types of networks.

\begin{figure*}[!t]
    \centering
    \subfloat[Cora]{\includegraphics[width=0.246\textwidth]{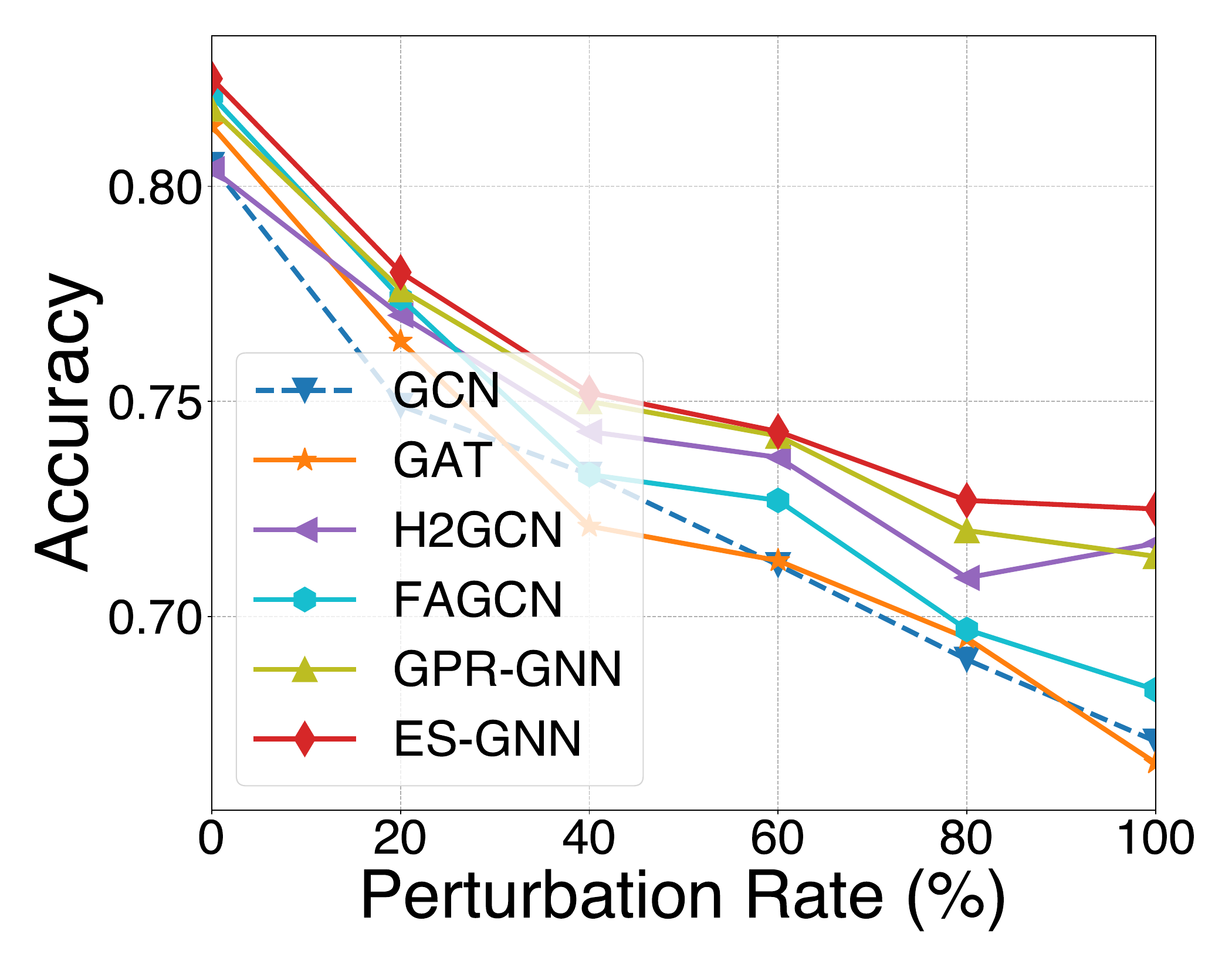}\label{fig:dfs_cora}}
    \hfil
    \subfloat[Citeseer]{\includegraphics[width=0.246\textwidth]{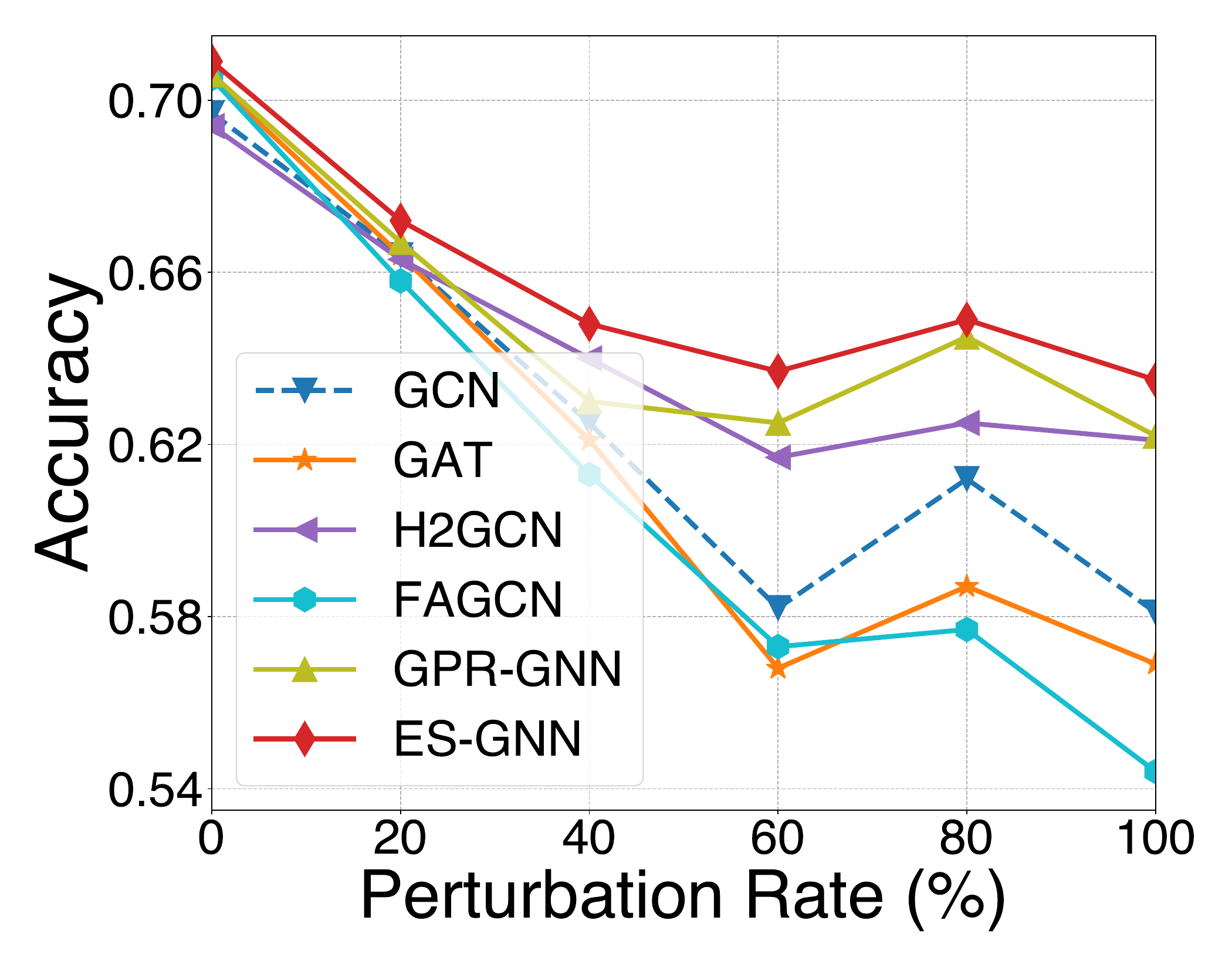}\label{fig:dfs_cite}}
    \hfil
    \subfloat[Pubmed]{\includegraphics[width=0.246\textwidth]{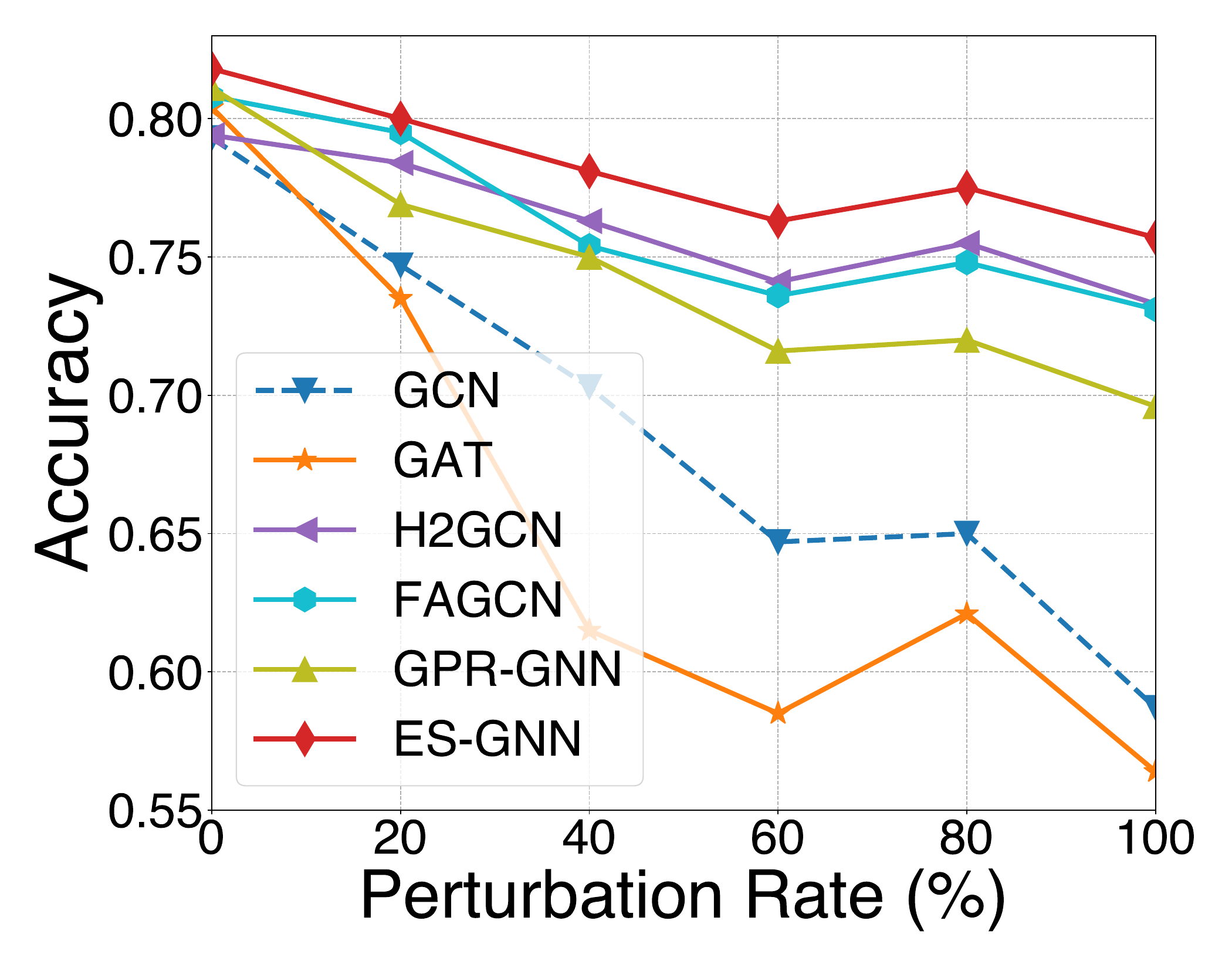}\label{fig:dfs_pub}}
    \hfil
    \subfloat[Polblogs]{\includegraphics[width=0.246\textwidth]{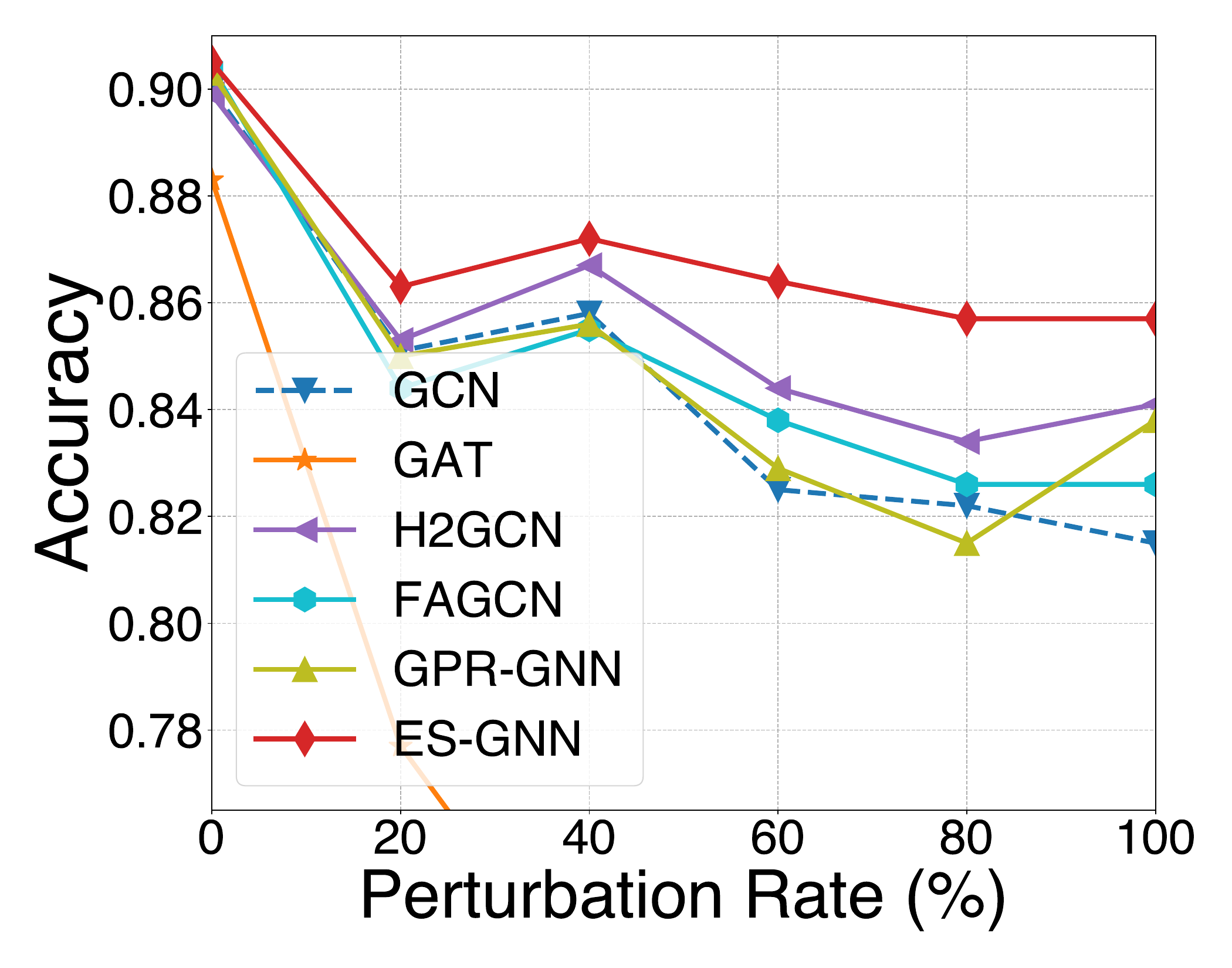}\label{fig:dfs_polblogs}}
    \caption{
    Results of different models on perturbed homophilic graphs. ES-GNN is able to identify the falsely injected (the task-irrelevant) graph edges, and exclude these connections from the final predictive learning, thereby displaying relative robust performance against adversarial edge attacks.
    }
    \label{fig:robust_dfs}
\end{figure*}

\subsection{Edge Analysis}\label{sec:analysis_es}
We analyze the split edges from our ES-layer using synthetic graphs as an example in this section. According to Section~\ref{sec:syn_data_gen}, the synthetic edges are defined as the task-relevant connections if they link nodes from the same class, and the task-irrelevant ones otherwise. Therefore, we calculate the percentages of heterophilic node connections, which are excluded from our task-relevant topology and disentangled in the task-irrelevant one, so as to investigate the discerning ability of ES-GNN between edges in different types. As can be observed in Table~\ref{tab:syn_edge_analysis}, 71.7\% task-irrelevant edges are identified on average across various homophily ratios. On the other hand, we also report the classification accuracies of ES-GNN and its variant while ablating ES-layer, from which approximately 10\% degradation can be observed. All of these strongly validate the effectiveness of our ES-layer and reasonably interprets the good performance of ES-GNN.

\subsection{Robustness Analysis}\label{sec:robust}
By splitting the original graph edge set into  task-relevant and task-irrelevant subsets, our proposed ES-GNN enjoys strong robustness particularly on homophilic graphs, since perturbed or noisy aspects of nodes could be purified from the task-relevant topology and disentangled in the task-irrelevant topology. To examine this, we randomly inject fake edges into graphs with perturbed rates from 0\% to 100\% with a step size of 20\%. Adversarially perturbed examples are generated from graphs with strong homophily, such as Cora, Citseer, Pubmed, and Polblogs. As shown in Fig.~\ref{fig:robust_dfs}, models considering graphs beyond homophily, i.e., H2GCN, FAGCN, GPR-GNN, and our model, consistently display a more robust behavior than GCN and GAT. That is mainly because fake edges may connect nodes across different labels, and consequently cause erroneous information sharing in the conventional methods. 

On the other hand, our ES-GNN beats all the baselines by an average margin of 2\% to 3\% on Citeseer, Pubmed, and Polblogs while displaying relatively the same results on Cora. We attribute this to the capability of our model in associating node connections with learning tasks. Take Pubmed dataset as an example. We investigate the learned task-relevant topologies and find that 81.0\%, 73.0\%, 82.1\%, 83.0\%, 82.6\% fake links get removed on adversatial graphs with perturbation rates from 20\% to 100\%. This also offers evidences supporting that our ES-layer is able to distinguish between task-relevant and irrelevant node connections. Therefore, despite a large number of false edge injections, the proximity information of nodes can still be reasonably mined in our model to predict their labels. Importantly, these empirical results also indicate that ES-GNN can still identify most of the task-irrelevant edges though no clear similarity or association between the connected nodes exists in the adversarial setting. 

\begin{figure}[t]
    \centering
    \subfloat[Cora]{\includegraphics[width=0.24\textwidth]{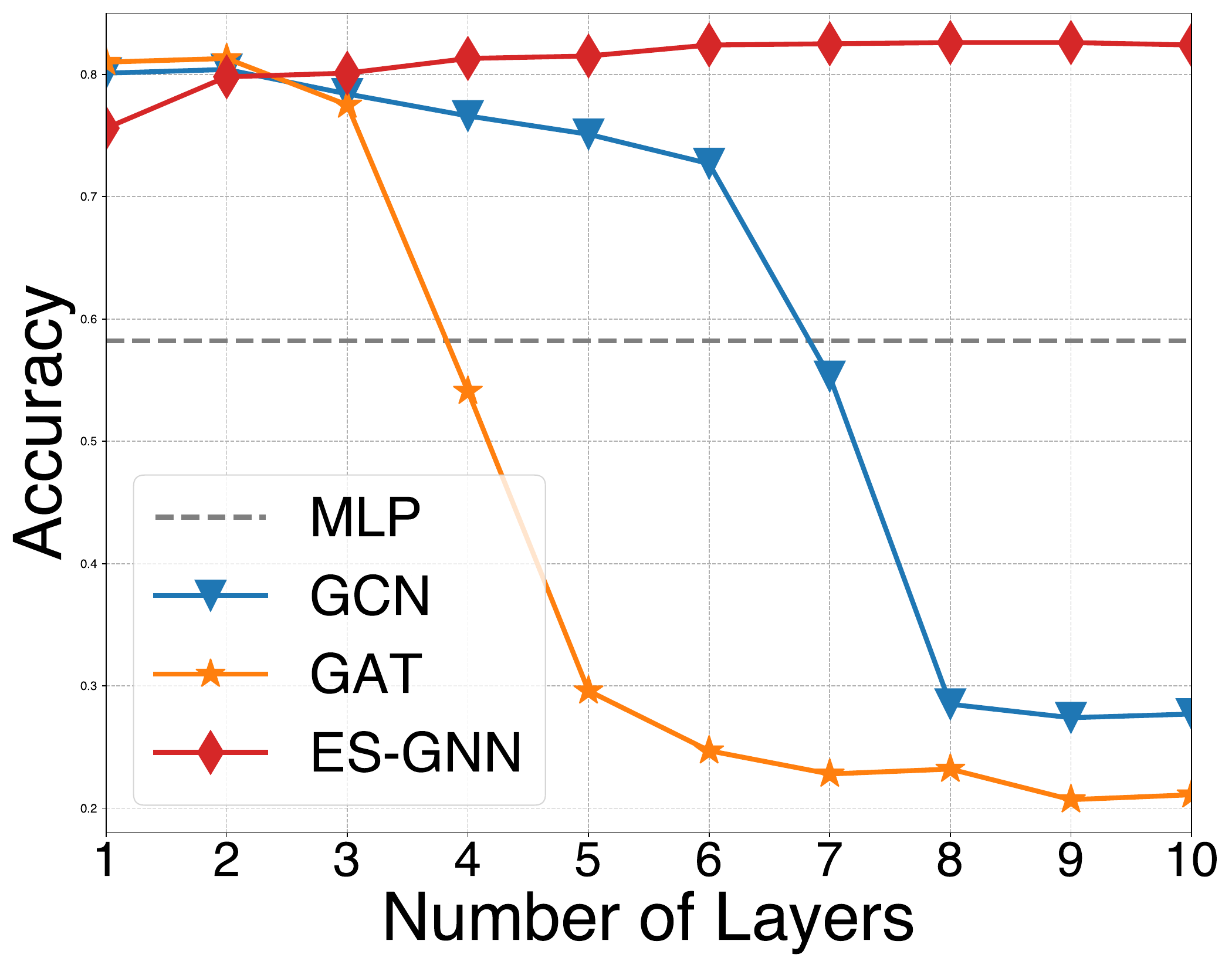}\label{fig:os_cora}}
    \subfloat[Citeseer]{\includegraphics[width=0.24\textwidth]{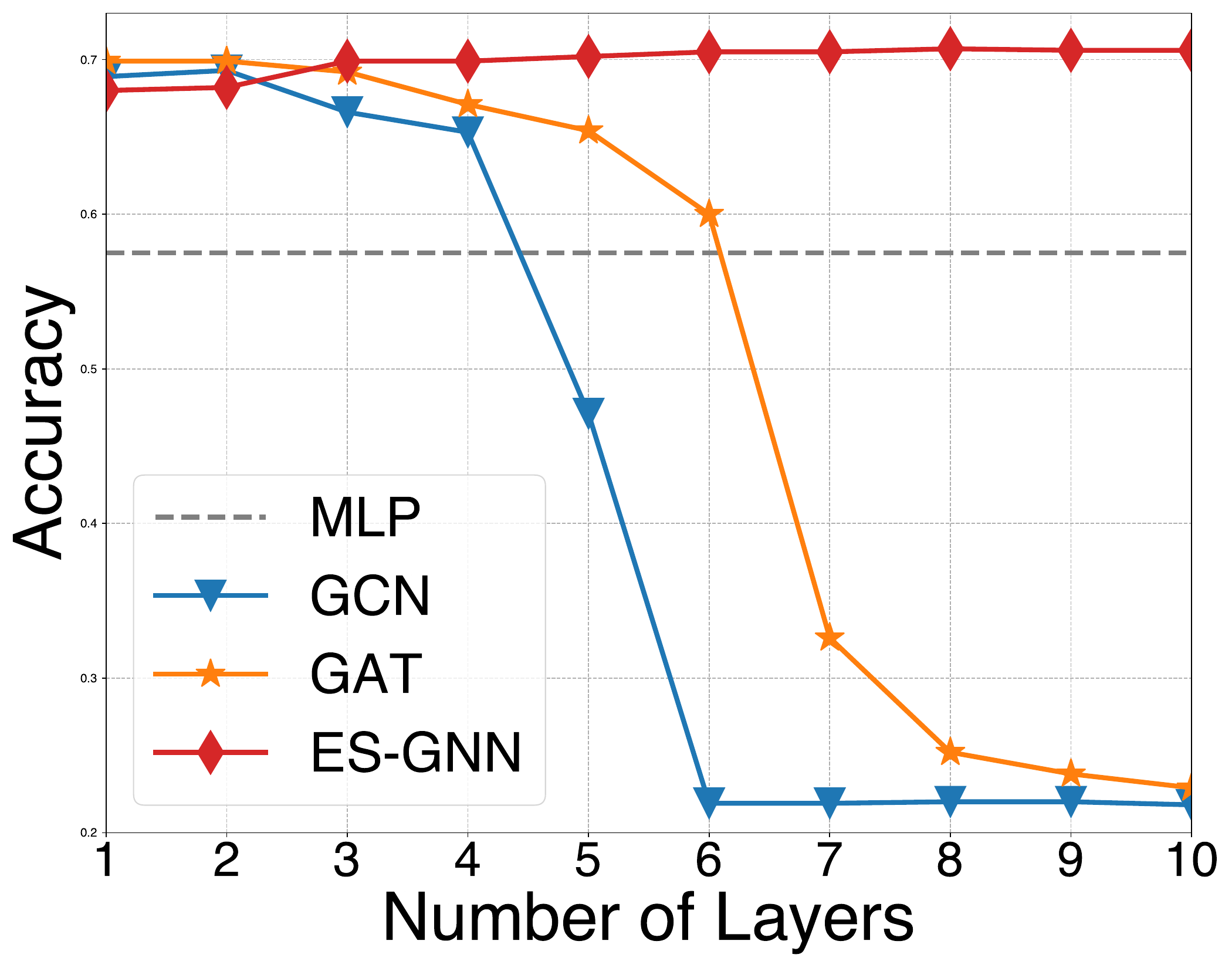}\label{fig:os_cite}}
    \hfill
    \subfloat[Pubmed]{\includegraphics[width=0.24\textwidth]{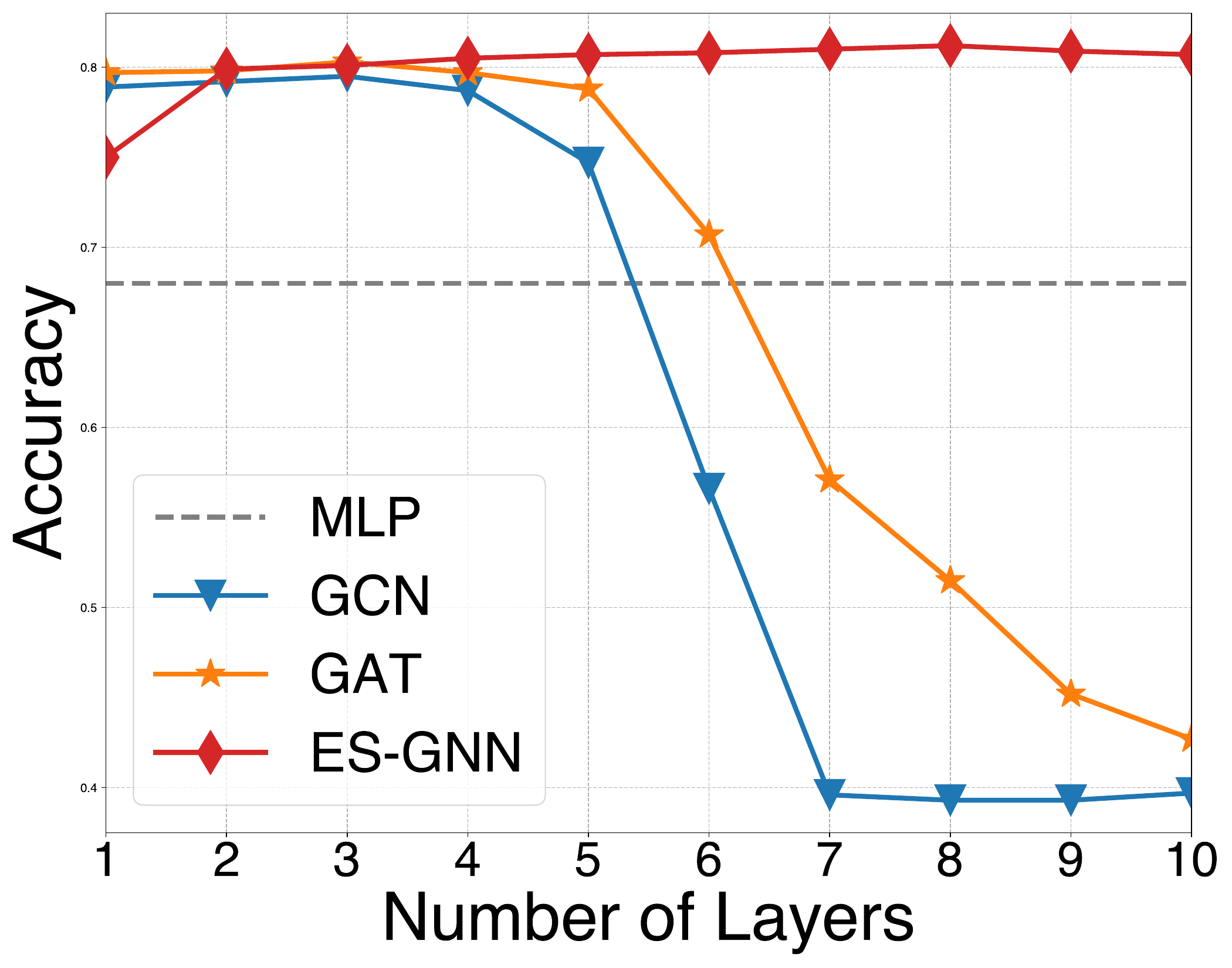}\label{fig:os_pubmed}}
    \subfloat[Polblogs]{\includegraphics[width=0.24\textwidth]{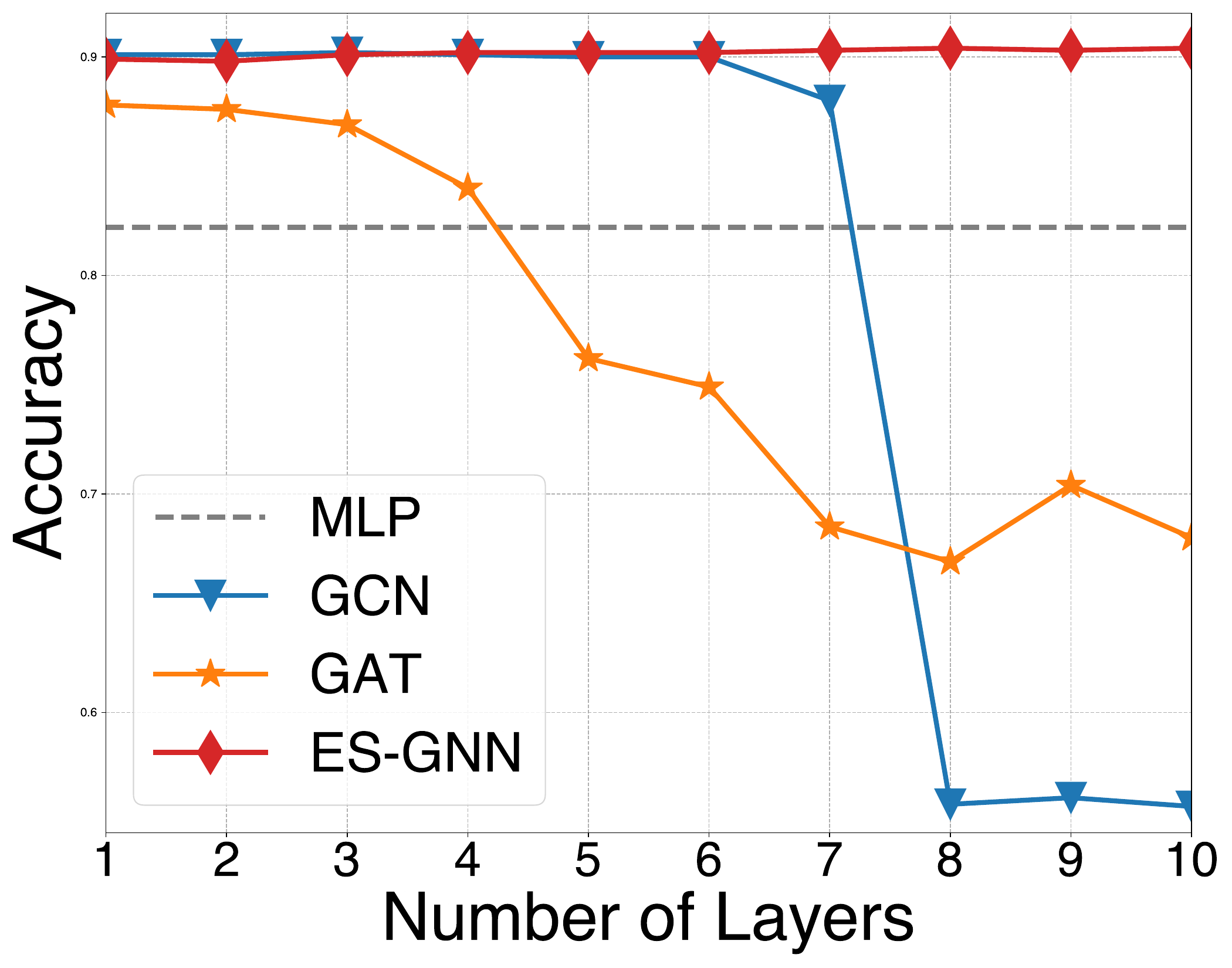}\label{fig:os__polblogs}}
    \caption{Classification accuracy vs. model depths.}
    \label{fig:os_prob}
\end{figure}

\begin{figure*}[t]
\centering
\includegraphics[width=0.85\textwidth]{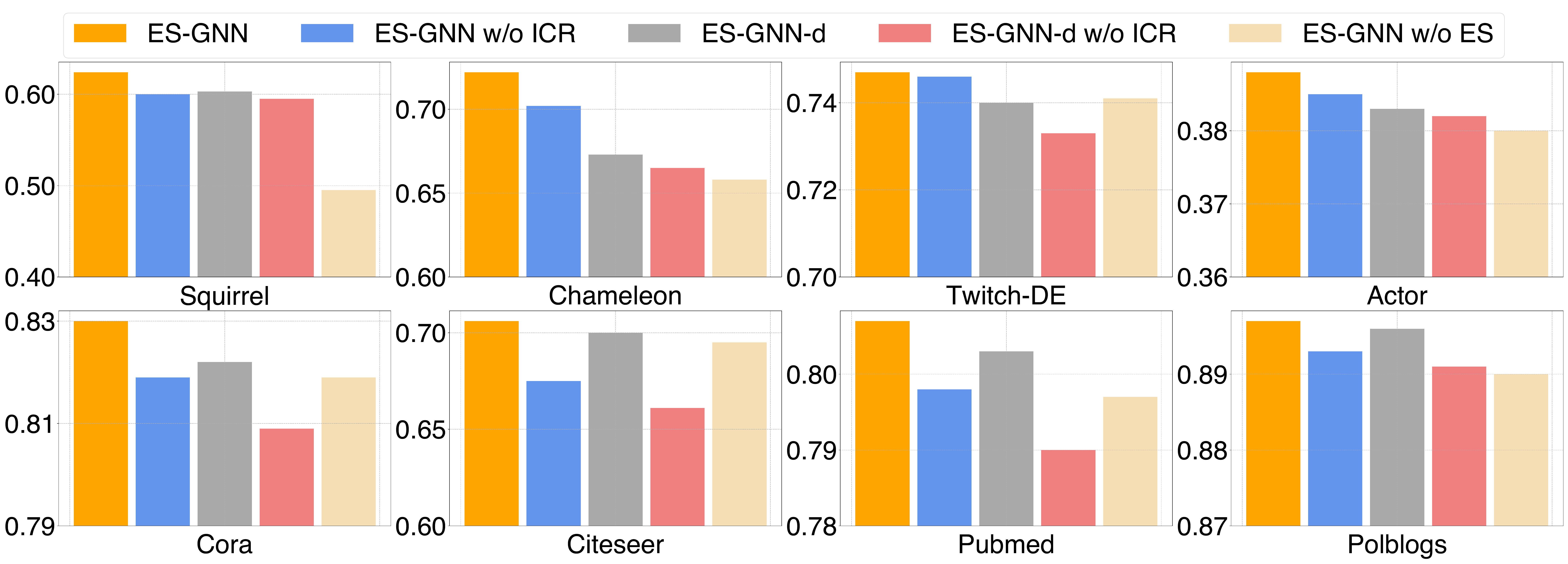}
\caption{Ablation study of ES-GNN on eight datasets in node classification.}
    \label{fig:abla_study}
\end{figure*}

\subsection{Alleviating Over-smoothing Problem}
In order to verify whether ES-GNN alleviates the over-smoothing problem, we compare it with GCN and GAT by varying the layer number in Fig.~\ref{fig:os_prob}. It can be observed that these two baselines attain their highest results when the number of layers reaches around two. As the layer goes deeper, the accuracies of both GCN and GAT gradually drop to a lower point. On the contrary, our ES-GNN presents a stable curve. In spite of starting from a relative lower point, the performance of ES-GNN keeps improving as the model depths increase, and eventually outperforms both GCN and GAT. The main reason is that, our ES-GNN can adaptively utilize proper graph edges in different layers to attain the task-optimal results with enlarged receptive fields. In other words, once an edge stops passing useful information or starts passing harmful messages, ES-GNN tends to identify it and remove it from learning the task-correlated representations, thereby having the ability of mitigating the over-smoothing problem.

\subsection{Channel Analysis and Ablation Study}
In this section, we compare ES-GNN with its variant ES-GNN-d which takes dual (both the task-relevant and irrelevant) channels for prediction, and perform an ablation study. Fig.~\ref{fig:abla_study} provides comparison on eight real-world datasets as examples. Here, we first specify some annotations including 1) ``w/o ICR'': without regularization loss $\mathcal{L}_{\text{ICR}}$, and 2) ``w/o ES'': without edge splitting (ES-) layer. Overall, two conclusions can be drawn from Fig.~\ref{fig:abla_study}. First, ES-GNN is consistently better than ES-GNN-d, implying that the task-irrelevant channels indeed capture some false information where model performance downgrades even with the doubled feature dimensions. Second, removing either ICR or ES-layer from both ES-GNN and ES-GNN-d leads to a clear accuracy drop. That validates the effectiveness of our model designs.

\subsection{Parameter Study}
This section presents the sensitivity analysis of hyper-parameters, specifically $\lambda_{\text{ICR}}$, $\epsilon_{\text{R}}$ and $\epsilon_{\text{IR}}$, using Chameleon and Cora datasets as typical examples. Fig.~\ref{fig:param_sensi} illustrates how varying these parameters affect our model's learning performance. Overall, we have the following observations: \textbf{1)} The effect of the regularization coefficient $\lambda_{\text{ICR}}$ is depicted in Fig.~\ref{fig:param_sensi}(a)-(b). For instance, the classification accuracy on the Chameleon dataset increases first and then gradually decreases. Favorable results can be attained by choosing $\lambda_{\text{ICR}}$ from $\left[1\mathrm{e}{-7},1\mathrm{e}{-5}\right]$. A similar trend can be also observed on Cora dataset where $\lambda_{\text{ICR}}$ is relatively robust within a wide albeit distinct interval; \textbf{2)} The influence of the scaling coefficients $\epsilon_{\text{R}}$ and $\epsilon_{\text{IR}}$ is evaluated by adjusting their values from $1\mathrm{e}{-3}$ to $1.0$ on Fig.~\ref{fig:param_sensi}(c)-(d) and Fig.~\ref{fig:param_sensi}(e)-(f), respectively. It can be observed that, despite variations in optimal settings, selecting values between $5\mathrm{e}{-2}$ and $0.5$ consistently yields promising performance.

\begin{figure}[!t]
    \centering
    \subfloat[Chameleon]{\includegraphics[width=0.235\textwidth]{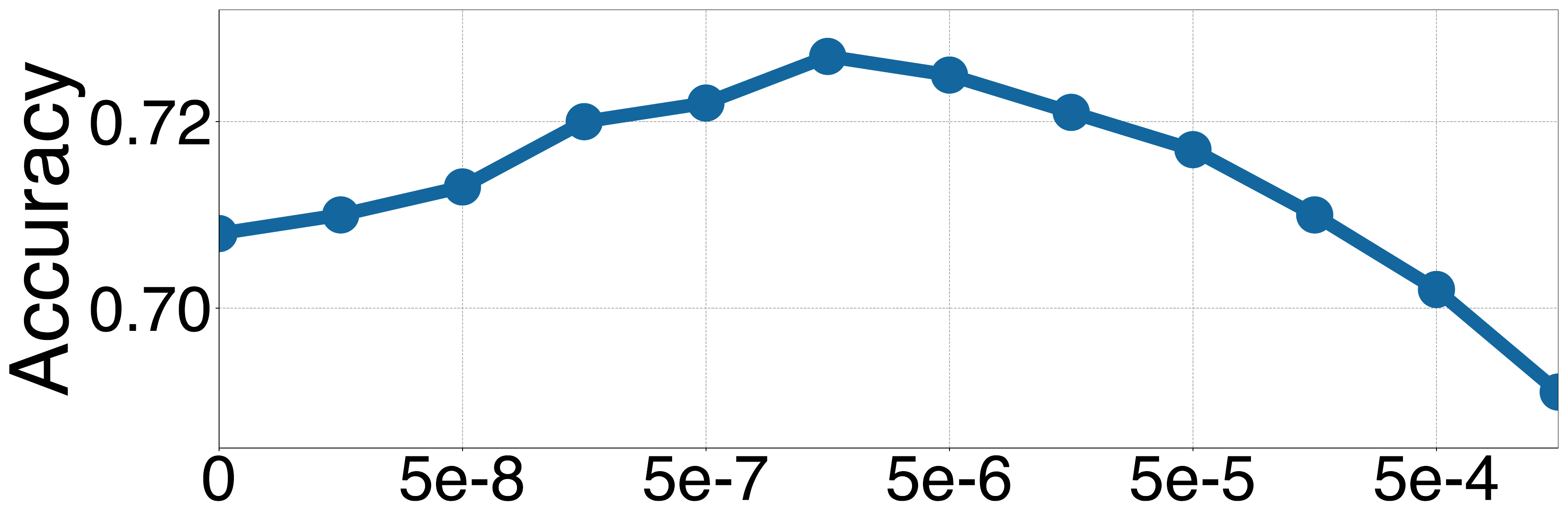}\label{fig:icr_cham}}
    \subfloat[Cora]{\includegraphics[width=0.235\textwidth]{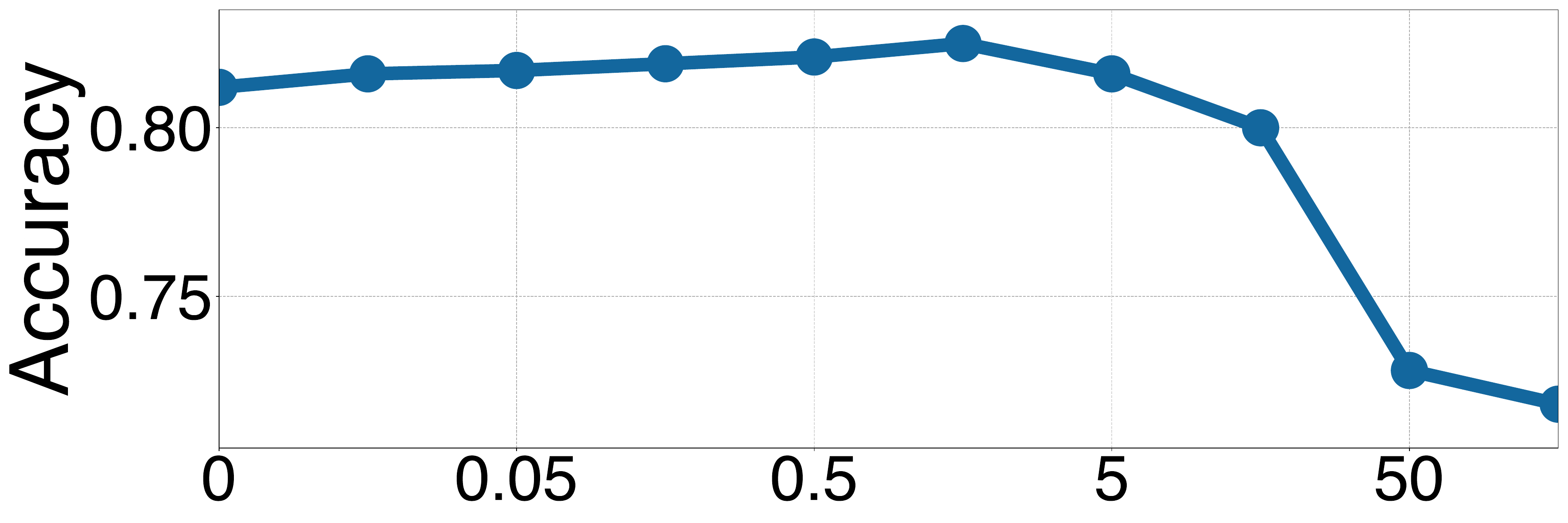}\label{fig:icr_cora}}
    \vfill
    \subfloat[Chameleon]{\includegraphics[width=0.235\textwidth]{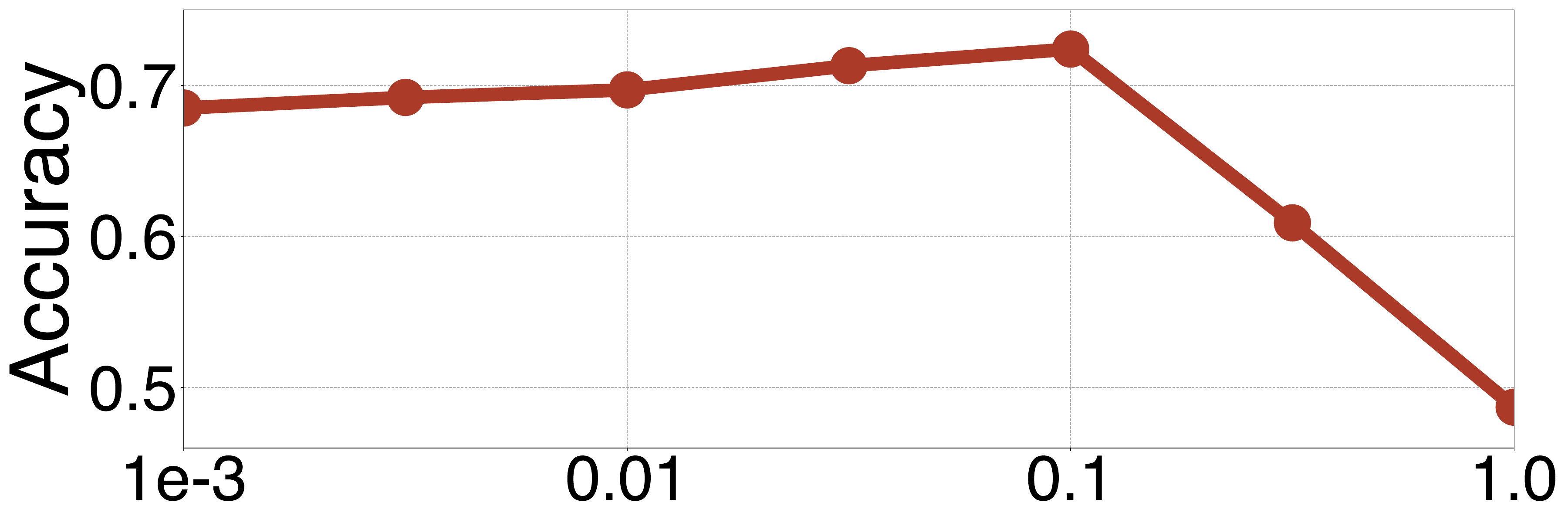}\label{fig:epsR_cham}}
    \subfloat[Cora]{\includegraphics[width=0.235\textwidth]{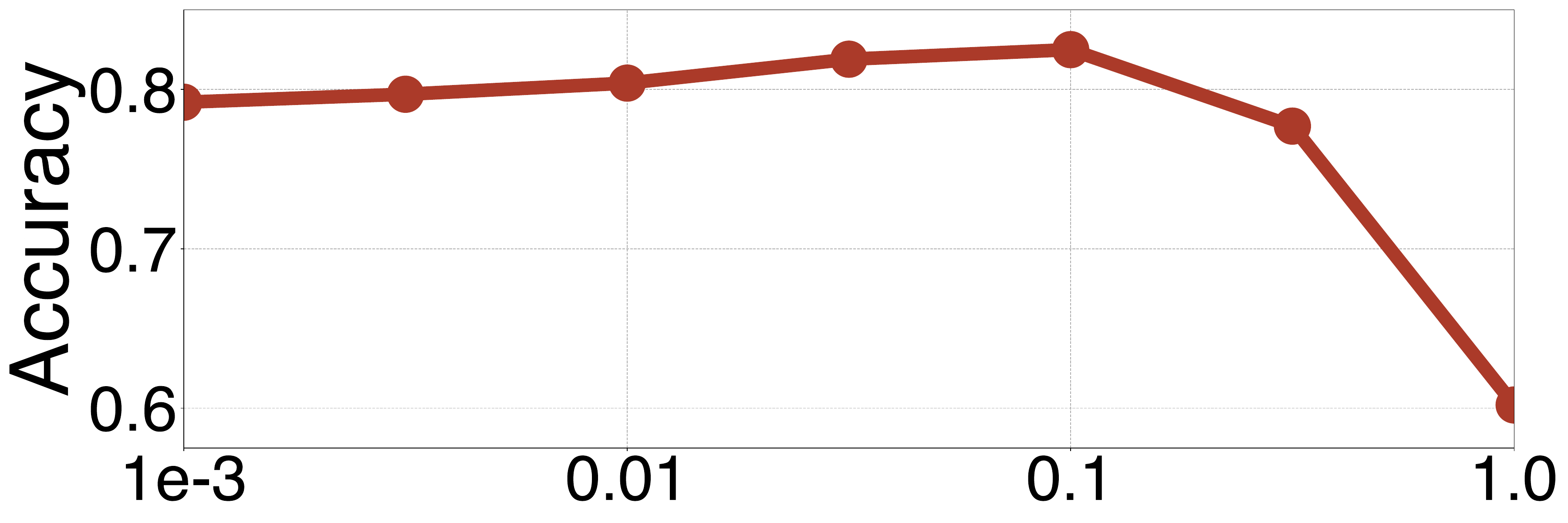}\label{fig:epsR_cora}}
    \vfill
    \subfloat[Chameleon]{\includegraphics[width=0.235\textwidth]{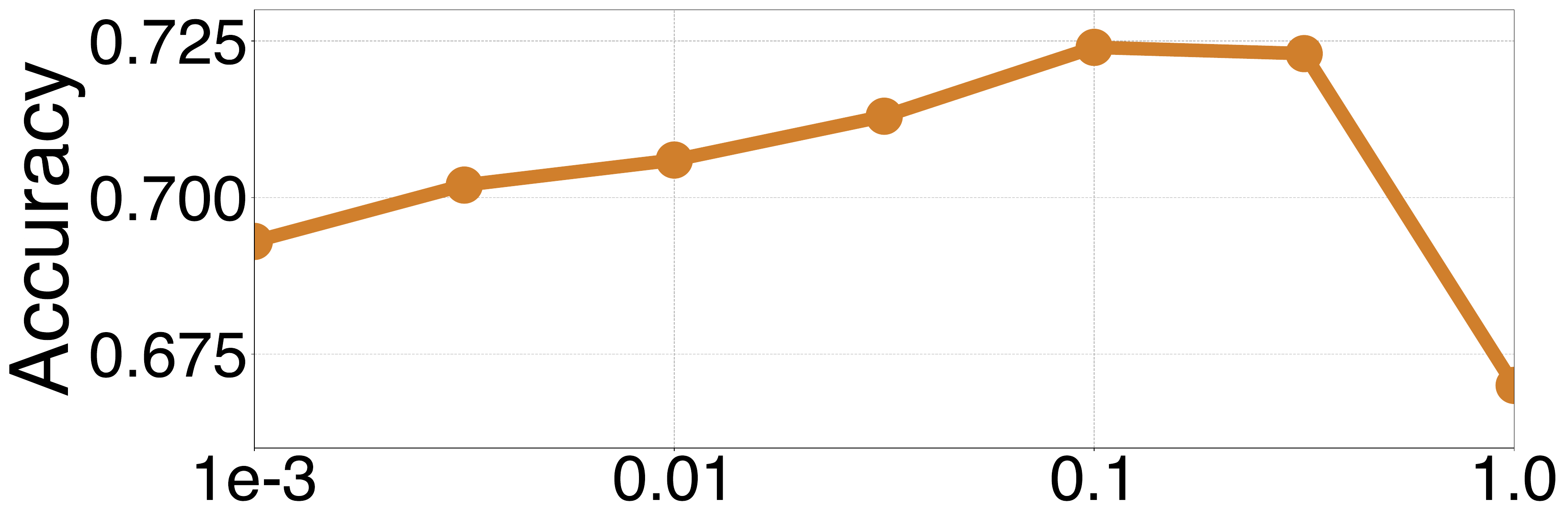}\label{fig:epsIR_cham}}
    \subfloat[Cora]{\includegraphics[width=0.235\textwidth]{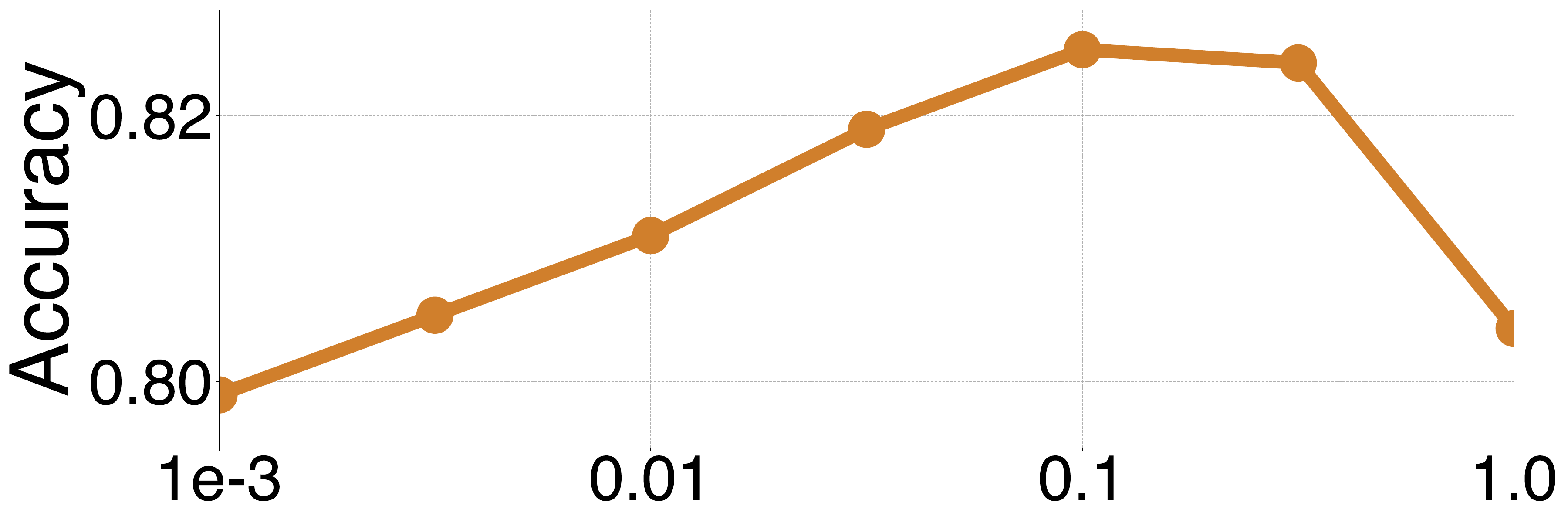}\label{fig:epsIR_cora}}
    \caption{Sensitivity analysis of hyper-parameters: $\lambda_{\text{ICR}}$, $\epsilon_{\text{R}}$ and $\epsilon_{\text{IR}}$ from top to bottom rows.
    }\label{fig:param_sensi}
\end{figure}

\section{Conclusion}
In this paper, we develop a novel graph learning framework that enables GNNs to go beyond the strong homophily assumption on graphs. We manage to establish a correlation between node connections and learning tasks through a plausible hypothesis, from which ES-GNN is derived with interpretable edge splitting. Our ES-GNN essentially partitions the original graph structure into task-relevant and irrelevant topologies as guide to disentangle node features, whereby the classification-harmful information can be disentangled and excluded from the final prediction target. Theoretical analysis illustrates our motivation and offers interpretations on the expressive power of ES-GNN on different types of networks. To provide empirical verification, we conduct extensive experiments over 11 benchmark and 1 synthetic datasets. The node classification results demonstrate the overall superior performance of our ES-GNN compared to 15 competitive baselines, which specialize in task-relevance, graph disentanglement and heterophily. In particular, we also conduct analysis on the split edges, correlation among disentangled features, model robustness, and the ablated variants. All of these results demonstrate the success of ES-GNN in identifying graph edges between different types, which also validates the effectiveness of our interpretable edge splitting. In future work, we will further explore more sophisticated designs in the edge splitting layer. Another promising direction would be how to extend our learning paradigm in accomplishing graph-level tasks.


\ifCLASSOPTIONcaptionsoff
  \newpage
\fi

\bibliographystyle{IEEEtran}
\bibliography{esgnn}

\vskip -2\baselineskip plus -1fil

\begin{IEEEbiography}[{\includegraphics[width=1in,height=1.25in,clip,keepaspectratio]{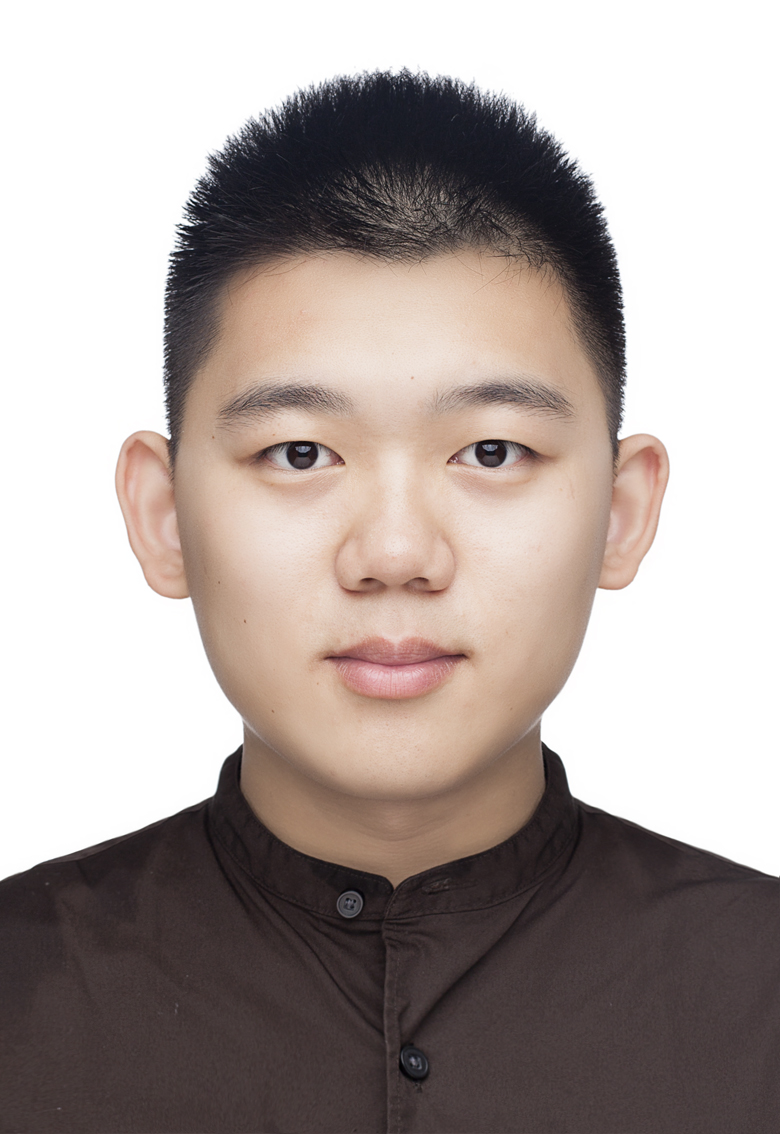}}]{Jingwei Guo}
received the First-class (Hons) degree in Applied Mathematics from University of Liverpool, UK, in 2018. After finishing his undergraduate study, he worked as a Research Associate at Xi’an Jiaotong-Liverpool University of China for a year. He is currently pursing his PhD degree at University of Liverpool, UK. His research focuses on developing new graph neural networks, and applying the techniques in various domains.
\end{IEEEbiography}

\vskip -2\baselineskip plus -1fil

\begin{IEEEbiography}[{\includegraphics[width=1in,height=1.25in,clip,keepaspectratio]{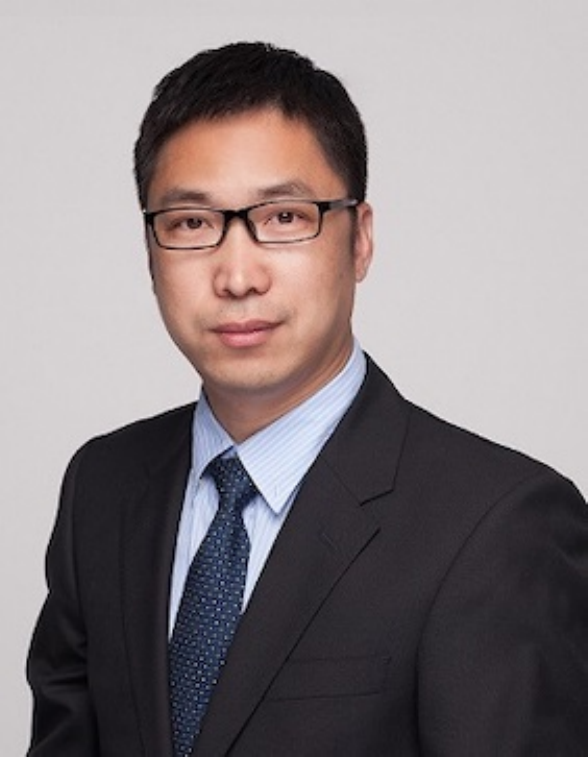}}]{Kaizhu Huang}
	 (corresponding author) 
	  works on machine learning, neural information processing, and pattern recognition. He is currently a tenured Professor of ECE at Duke Kunshan University (DKU). Prof. Huang obtained his PhD degree from Chinese University of Hong Kong (CUHK) in 2004. He worked in Fujitsu Research Centre, CUHK, University of Bristol, National Laboratory of Pattern Recognition, Chinese Academy of Sciences, and Xi’an Jiaotong-Liverpool University from 2004 to 2022. He was the recipient of 2011 Asia Pacific Neural Network Society Young Researcher Award. He received best (runner-up) paper or book awards nine times and published extensively  in journals (IEEE T-NNLS, IEEE T-IP, IEEE T-PAMI, IEEE T-CYB) and conferences (AAAI, ICML, CIKM, ECML, ICCV, CVPR, NeurIPS, ICDM). He serves as associated editors/advisory board members in a number of journals and book series. He was invited as keynote speaker at more than 40 international conferences or workshops.
\end{IEEEbiography}

\vskip -2\baselineskip plus -1fil

\begin{IEEEbiography}[{\includegraphics[width=1in,height=1.25in,clip,keepaspectratio]{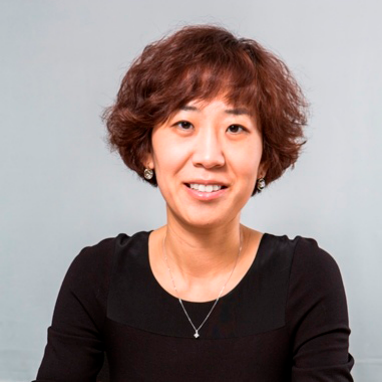}}]{Rui Zhang}
received the First-class (Hons) degree in Telecommunication Engineering from Jilin University of China in 2001 and the Ph.D. degree in Computer Science and Mathematics from University of Ulster, UK in 2007. After finishing her PhD study, she worked as a Research Associate at University of Bradford and University of Bristol in the UK for 5 years. She joined Xi’an Jiaotong-Liverpool University in 2012 and currently holds the position of Senior Associate Professor. Her  research interests include machine learning, data mining and statistical analysis.
\end{IEEEbiography}

\vskip -2\baselineskip plus -1fil

\begin{IEEEbiography}[{\includegraphics[width=1in,height=1.25in,clip,keepaspectratio]{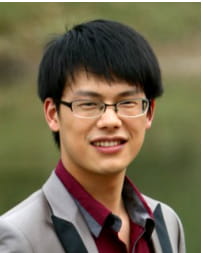}}]{Xinping Yi} received the Ph.D. degree in electronics and communications from T\'el\'ecom ParisTech, Paris, France, in 2015. He has been a Professor at Southeast University, Nanjing, China, since 2023. Prior to that, he was a Lecturer (Assistant Professor) with University of Liverpool, Liverpool, UK, a Research Associate with Technische Universitat Berlin, Berlin, Germany, and a Research Assistant with EURECOM, Sophia Antipolis, France. His main research interests include network information theory, trustworthy artificial intelligence, graph machine learning, and their applications in wireless communications.
\end{IEEEbiography}

\end{document}